\documentclass{article}

\usepackage{iftex}
\ifPDFTeX
  \usepackage[utf8]{inputenc}
  \usepackage[T1]{fontenc}
\fi
\usepackage[preprint]{neurips_2026}

\usepackage{iftex}

\usepackage{microtype}
\usepackage{amsmath,amssymb,amsthm,mathtools}
\usepackage{booktabs}
\usepackage{multirow}
\usepackage{graphicx}
\usepackage[table]{xcolor}
\usepackage{enumitem}
\usepackage{hyperref}
\usepackage[nameinlink,noabbrev]{cleveref}
\usepackage{array}
\usepackage{caption}
\usepackage{subcaption}
\usepackage{tikz}
\usepackage{amsthm}

\usepackage{natbib}

\graphicspath{{figures/method/}{figures/experiments/}{figures/cases/}{figures/appendix/}}
\hypersetup{
  colorlinks=true,
  linkcolor=blue,
  citecolor=blue,
  urlcolor=blue
}

\newtheorem{proposition}{Proposition}[section]

\theoremstyle{remark}

\usepackage{alltt,capt-of}
\usepackage{tcolorbox}
\newenvironment{qbox}
{\begin{tcolorbox}[colback=white, width=0.95\linewidth, center, left=1pt,right=1pt,top=1pt,bottom=1pt]}
{\end{tcolorbox}}

\usepackage{colortbl}
\usepackage[misc]{ifsym}

\title{Can Attribution Predict Risk? From Multi-View Attribution to Planning Risk Signals in End-to-End Autonomous Driving}

\author{%
    Le Yang$^{1}$, Haijun Liu$^{1}$, Jiawei Liang$^{1}$, ShangQuan Sun$^{2}$, Xiaochun Cao$^{1}$ \\
    \small$^{1}$Sun Yat-sen University~~~~$^2$Nanyang Technological University\\ 
    \small \texttt{yangle26@mail2.sysu.edu.cn}~~~~~~~~\texttt{caoxiaochun@mail.sysu.edu.cn} \\
}

\begin{document}

\maketitle

\begin{abstract}
    End-to-end autonomous driving models generate future trajectories from multi-view inputs, improving system integration but introducing opaque decisions and hard-to-localize risks. Existing methods either rely on auxiliary monitoring models or generate textual explanations, but are decoupled from the planning process and fail to reveal the visual evidence underlying trajectory generation. While attribution offers a direct alternative, planning differs from image classification by taking six-view camera images as input and predicting continuous multi-step trajectories, requiring attribution to capture both critical views and regions and their influence on outputs. Moreover, whether attribution maps can support risk identification remains underexplored. 
    To address this, we propose a hierarchical attribution framework for end-to-end planning. Specifically, using L2 consistency with the original trajectory as the objective, we design a coarse-to-fine region attribution strategy that searches candidate regions across the full six-view input and refines attribution within them. We further extract three attribution statistics as predictive signals for planning risk, including attribution entropy to measure how concentrated the planner's reliance is over the joint visual space, within-camera spatial variance to characterize how spread out the attribution is within each view, and cross-camera Gini coefficient to quantify how unevenly attribution is distributed across the six cameras. Experiments on BridgeAD, UniAD, and GenAD show that these statistics correlate with planning risk, achieving Spearman correlations of $0.30 \pm 0.07$ with trajectory error and AUROC of $0.77 \pm 0.04$ for collision detection. The signal generalizes to held-out scenes with negligible degradation and remains stable under an alternative attribution baseline.
\end{abstract}

\section{Introduction}


End-to-end autonomous driving has recently gained increasing attention by directly mapping raw sensor inputs to planning outputs and reducing the dependence on hand-crafted modular pipelines~\cite{hu2023uniad, jiang2023vad, sun2024sparsedrive, zhang2025bridgead}. Despite its promise, safety remains a major obstacle to real-world deployment, as planning errors may directly result in unsafe maneuvers, collisions, or other high-risk driving behaviors~\cite{leong2025steering,liang2025interaction}. As end-to-end systems are increasingly explored in industrial scenarios, such as Tesla FSD V12~\cite{elluswamy2024fsd}, it becomes critical to understand what planning decisions rely on, not only whether a trajectory is wrong, but also what visual evidence supports it and whether such reliance is stable~\cite{chen2024e2esurvey,kuznietsov2024safex}.

Interpretability in autonomous driving helps reveal the visual evidence and decision logic underlying planning behaviors, thereby enhancing the transparency and trustworthiness of autonomous driving systems~\cite{kuznietsov2024safex,wen2024road}. Existing methods either generate language-based explanations~\cite{sima2024drivelm,xu2025drivegpt4v2,zhou2025autovla} or train auxiliary monitors~\cite{kuznietsov2024safex,keser2023interpretable} for risk detection. However, language explanations remain opaque in their reasoning process~\cite{turpin2023unfaithful,barez2025cot_not_xai,xie2025drivebench}, while auxiliary monitors are decoupled from the planner and mainly capture scene-level rather than decision-level risk. Since dangerous scenes do not necessarily cause planner failures~\cite{wang2025redoubt}, and benign-looking scenes may still induce unsafe decisions, risk should be analyzed from the planner's own behavior~\cite{antonante2023task,westhofen2023criticality}.

Attribution~\cite{petsiuk2018rise,chen2024less,chen2025vps} offers a natural way to analyze decision-level risk by linking planner outputs to visual evidence. Prior studies further suggest that attribution drift can monitor model performance~\cite{nigenda2022amazon,wang2023feature,decker2024explanatory}. However, existing attribution methods for autonomous driving~\cite{karim2022toward,cultrera2020explaining} mainly treat attribution as a post-hoc explanation or visualization tool~\cite{chen2024e2esurvey}, rather than as a signal for risk prediction. Although some methods use attribution~\cite{hacker2023insufficiency} or attention~\cite{stocco2022thirdeye} for risk monitoring, they are not tailored to end-to-end planning, where modern planners~\cite{zhang2025bridgead,hu2023uniad} consume six-view camera images and produce continuous multi-step trajectories instead of discrete classification scores or simple control outputs. This raises the following question:

\begin{qbox}
    \begin{center}
    \small \textbf{\textit{Can attribution not only explain multi-view end-to-end planning, but also predict potential risk?}}
    \end{center}
\end{qbox}

Motivated by this, we propose a hierarchical attribution framework for multi-view end-to-end driving planners. Based on high-fidelity subset-selection attribution~\cite{chen2024less,chen2025vps,chen2025eagle}, we treat the six camera views as a unified attribution space and use L2 trajectory consistency, rather than class confidence, as the attribution objective. To improve efficiency, we introduce a coarse-to-fine search that globally ranks region groups across all views and then refines subregion-level attributions within each group. This produces faithful attribution maps that connect trajectory outputs to visual evidence. 


Building on this, we investigate whether attribution maps can reflect the planner’s decision-level risk. From a heuristic perspective, robust planning should avoid over-reliance on limited visual evidence, since dominance by a small region, localized area, or single view may make trajectories sensitive to perturbations or occlusions. We therefore derive three complementary statistics: attribution entropy for global concentration over the six-view space, within-camera spatial variance for localized concentration within each view, and cross-camera Gini coefficient for imbalance across views. Lower entropy, lower spatial variance, and higher Gini indicate stronger global, within-view, and cross-view over-reliance, respectively, and may serve as attribution-based warning signals for planning risk.


Experiments are conducted on the validation split of the nuScenes dataset~\cite{caesar2020nuscenes} with three representative end-to-end autonomous driving models: BridgeAD~\cite{zhang2025bridgead}, UniAD~\cite{hu2023uniad}, and GenAD~\cite{zheng2024genad}. We consider trajectory error and predicted collision as two types of planning risk. The three attribution statistics are fitted with ridge regression for trajectory error and logistic regression for collision prediction, and evaluated using Spearman correlation and AUROC, respectively. When both fit and evaluated on the full validation set, the statistics achieve an AUROC of $0.77 \pm 0.04$ across the three planners for collision prediction, indicating that attribution maps contain signals correlated with collision risk. To assess generalization, we further conduct 20 random splits, fitting the statistics on 80\% of the scenes and evaluating on the held-out 20\%. The resulting AUROC reaches $0.77 \pm 0.09$ across the three planners, demonstrating that the attribution-based risk signal transfers to unseen scenes with negligible degradation. Finally, replacing our attribution method with RISE~\cite{petsiuk2018rise} gives a weaker but consistent signal across planners, with an AUROC of $0.67 \pm 0.11$, suggesting that the proposed risk statistics are not algorithm-specific.
 In summary, our contributions are as follows:
\begin{itemize}
    \item We investigate whether attribution can go beyond post-hoc explanation and serve as a predictive signal for planning risk in end-to-end autonomous driving.
    \item A coarse-to-fine attribution method for multi-view planners, treating six camera views as a unified search space and using L2 trajectory consistency as the attribution objective.
    \item Three attribution statistics that quantify global, within-view, and cross-view over-reliance as predictive signals for planning risk.
    \item Experiments on nuScenes with BridgeAD, UniAD, and GenAD validate risk correlation, collision prediction, and unseen-sample generalization.
\end{itemize}

\section{Related Work}
 
\textbf{Explanation in End-to-End Autonomous Driving:} 
Some autonomous driving models improve transparency by generating textual explanations~\cite{sima2024drivelm,xu2025drivegpt4v2,zhou2025autovla}, such as DriveGPT4~\cite{xu2024drivegpt4,xu2025drivegpt4v2}, which uses multimodal large language models for language-based driving reasoning. Other methods enhance interpretability through object-level relevance analysis~\cite{renz2022plant} or attribution-based explanations~\cite{karim2022toward,cultrera2020explaining} for simplified classification-based driving models. However, these methods mainly provide post-hoc explanations rather than predictive warning signals for potential planning risks.

\textbf{Attribution for Risk Monitoring:} INFORMER~\cite{shu2024informer} leverages inter-attribution-map relationships to generate risk warnings in medical image diagnosis, and Hacker \textit{et al.}~\cite{hacker2023insufficiency} use attribution-prototype matching to warn against unreliable traffic sign recognition. However, these methods are mainly developed for relatively constrained classification settings, where the prediction target is a discrete label and the visual input is much simpler than multi-view driving scenes. ThirdEye~\cite{stocco2022thirdeye} uses attention-map entropy to warn against unreliable decisions, but attention is not necessarily faithful to the model's actual decision evidence, especially in modern multi-component end-to-end planners where internal attention may be decoupled from spatial visual evidence. In contrast, we target end-to-end driving planning by designing a faithful multi-view attribution algorithm and deriving three statistics to analyze how attribution patterns relate to decision-level planning risk.

\section{Hierarchical Region Attribution}
\label{sec:method}

\begin{figure*}[t]
\centering
\includegraphics[width=\linewidth]{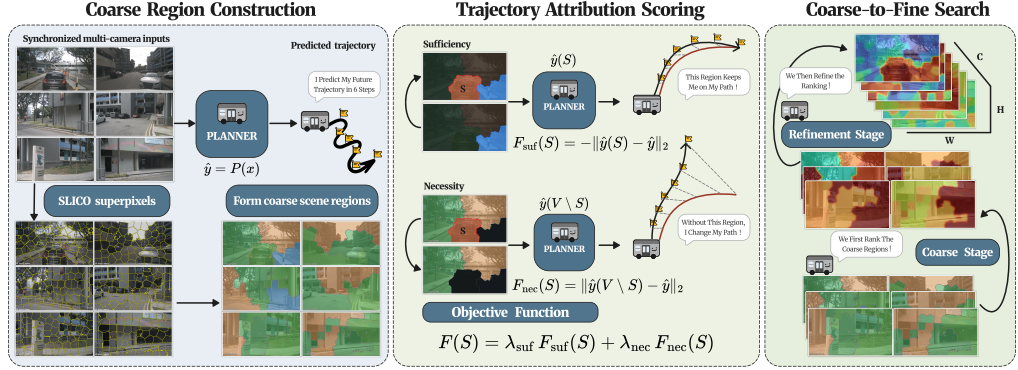}
\caption{\textbf{Hierarchical attribution pipeline.} Multi-camera inputs are partitioned into SLICO superpixels and grouped into coarse regions (left). Candidate subsets are evaluated by an objective function defined on the planner’s predicted trajectory (middle). A coarse-to-fine greedy search first selects coarse regions and then refines subregions, producing the per-pixel saliency tensor (right).}
\label{fig:method-pipeline}
\end{figure*}

\subsection{Task Formulation}

Consider a multi-camera end-to-end planner $P$ that takes synchronized images $x=(x^{(1)},\dots,x^{(C)})$ from $C$ cameras and outputs a planned trajectory $\hat y=P(x)$. Our interpretability analysis asks a single question: which visual regions does the planner rely on to produce this particular trajectory? We therefore focus on subregions that contribute to the output trajectory. Using SLICO superpixel segmentation~\cite{achanta2012slic}, we partition each camera view $x^{(c)}$ into subregions $V^{(c)}=\{v^{(c)}_1,\dots,v^{(c)}_{N_c}\}$ and merge all views into a joint candidate set
\[
  V=\bigsqcup_{c=1}^{C}V^{(c)}.
\]
We cast attribution as a subset selection problem~\cite{chen2024less,chen2025vps}:
\begin{equation}
  \max_{\pi\in\mathcal{P}(V),\ |\pi|\le K}
  \sum_{r=1}^{|\pi|} F(\pi_{1:r}),
  \label{eq:ordered}
\end{equation}
where $\mathcal{P}(V)$ is the set of all ordered subsets of $V$, $\pi$ is one such ordered sequence, $\pi_{1:r}$ its length-$r$ prefix, $K$ the maximum explanation length, and $F(\cdot)$ a set function measuring how well a set of subregions explains the current planning decision. This objective scores both the final selected set and every intermediate prefix as the explanation is built. The remaining question is how to design $F(\cdot)$ and optimize it efficiently.

\subsection{Objective Function}

We design a lightweight and faithful attribution algorithm for explaining trajectories produced by end-to-end driving planners. Building on recent region-selection interpretability work~\cite{chen2025vps,chen2025eagle}, we instantiate~\eqref{eq:ordered} with a set scoring function $F$ based on \emph{sufficiency} and \emph{necessity}, and rank regions by marginal-gain search. Motivated by classical greedy set selection~\cite{nemhauser1978analysis}, this procedure prioritizes regions with the largest incremental contribution, while not assuming that the neural-network-induced objective is strictly submodular. Specifically, $F$ enforces two complementary criteria: when the selected regions are kept, the planner's decision should remain close to that on the full input, when they are removed, the decision should change substantially.

\textbf{Sufficiency score.}
If keeping only the regions in $S$ still produces a trajectory close to $\hat y$ on the full input, then $S$ has captured this critical evidence. For a candidate subset $S$, we define the sufficiency score as
\begin{equation}
  F_{\mathrm{suf}}(S)=-\|\hat y(S)-\hat y\|_2,
  \label{eq:sufficiency}
\end{equation}
where $x_S$ retains the pixels covered by $S$ and replaces the rest with a baseline value, and $\hat y(S)=P(x_S)$ is the planner's output on this perturbed input. The closer this value is to zero, the better $S$ alone reproduces the planned trajectory.

\textbf{Necessity score.}
 If removing $S$ from the input causes the trajectory to deviate substantially, then $S$ is what the decision actually rests on. For a candidate subset $S$, we define the necessity score as
\begin{equation}
  F_{\mathrm{nec}}(S)=\|\hat y(V\setminus S)-\hat y\|_2,
  \label{eq:necessity}
\end{equation}
where $\hat y(V\setminus S)=P(x_{V\setminus S})$ is the planning output after $S$ has been removed. Larger values indicate that the planner cannot maintain its original behavior in the absence of $S$.

\textbf{Combined objective.}
A subset that scores high on only one criterion can mislead: high sufficiency with low necessity may simply mark redundant background, while high necessity with low sufficiency may reflect a generic perturbation that carries little explanatory content. A linear combination of the two jointly characterizes the sufficiency and necessity of an end-to-end planner's trajectory:
\begin{equation}
  F(S)=\lambda_{\mathrm{suf}}\,F_{\mathrm{suf}}(S)+\lambda_{\mathrm{nec}}\,F_{\mathrm{nec}}(S).
  \label{eq:objective}
\end{equation}

\subsection{Hierarchical Coarse-to-Fine Search}
\label{sec:hier-search}

Optimizing~\eqref{eq:ordered} under~\eqref{eq:objective} is NP-hard. Prior work~\cite{chen2024less} greedily adds the region with the largest marginal gain at each step, but a forward pass through an end-to-end driving planner is far more expensive than through the classifiers or detectors used in earlier post-hoc attribution~\cite{chen2024less,chen2025vps}, making exhaustive evaluation over the joint multi-camera space prohibitive. We observe that driving planners first attend to a few spatially adjacent, semantically related scene parts, such as the upcoming intersection, an adjacent lane, or the region around the leading vehicle~\cite{chitta2022transfuser,shao2023reasonnet}, and only then to finer cues within them, such as lane boundaries, agent contours, or traffic signals~\cite{jiang2023vad,chitta2021neat}. Attribution search can mirror this hierarchy. We therefore propose a \emph{coarse-to-fine search}: aggregate adjacent subregions into \emph{groups}, rank groups globally at low cost (\emph{coarse stage}), then use this ranking as a prior for ranking subregions within each group (\emph{refinement stage}).

\textbf{Coarse stage.}
We merge spatially adjacent subregions into groups and run the coarse-stage greedy search over these groups, yielding an ordered sequence $G_{(1)},G_{(2)},\dots,G_{(L_{\mathrm{c}})}$ of length $L_{\mathrm{c}}$ together with marginal gains $d_1,d_2,\dots,d_{L_{\mathrm{c}}}$, where
\[
  d_j=F\!\Bigl(\bigcup_{i\le j}G_{(i)}\Bigr)-F\!\Bigl(\bigcup_{i<j}G_{(i)}\Bigr).
\]
Note that the coarse stage uses the \emph{same} objective $F$ as subregion-level search. Only the search unit changes, from individual regions to groups. The corresponding group-level objective is $F_{\mathrm{grp}}(\mathcal{G})\triangleq F\bigl(\bigcup_{G\in\mathcal{G}}G\bigr)$. Under an ideal submodular objective, grouping does not break the selection structure of marginal-gain greedy search (see ~\Cref{prop:coarse-stage-grouped-submodularity} in the appendix).

\textbf{Refinement stage.}
Each group from the coarse stage now enters refinement. For group $G_{(j)}$, refinement does not restart from the empty set. Instead, it uses the union of all preceding groups as a conditional prefix,
\[
  \mathcal{C}_j=\bigcup_{i<j}G_{(i)}.
\]
The algorithm then runs greedy ranking over the subregions in $G_{(j)}$, at each step selecting
\[
  \arg\max_{v\in G_{(j)}\setminus S_{\mathrm{local}}}\Delta(v\mid\mathcal{C}_j\cup S_{\mathrm{local}}),
\]
where $S_{\mathrm{local}}$ is the set of subregions already selected inside the group. The advantage is that the coarse stage's inter-group ordering is preserved in full as prior information: because the relative order of groups stays fixed, every local search inherits the preceding groups as its conditional prefix. Refinement within each group, therefore, proceeds independently and naturally parallelizes in implementation. Finally, we take each region's marginal gain in the sequence as its attribution score and map it back to the pixel locations in the corresponding camera image, yielding the planning saliency tensor $T\in\mathbb{R}^{C\times H\times W}$. The attribution statistics in the next section are built directly on $T$.

\section{From Attribution Maps to Risk Signals}
\label{sec:statistics}

\begin{figure*}[t]
\centering
\includegraphics[width=\linewidth]{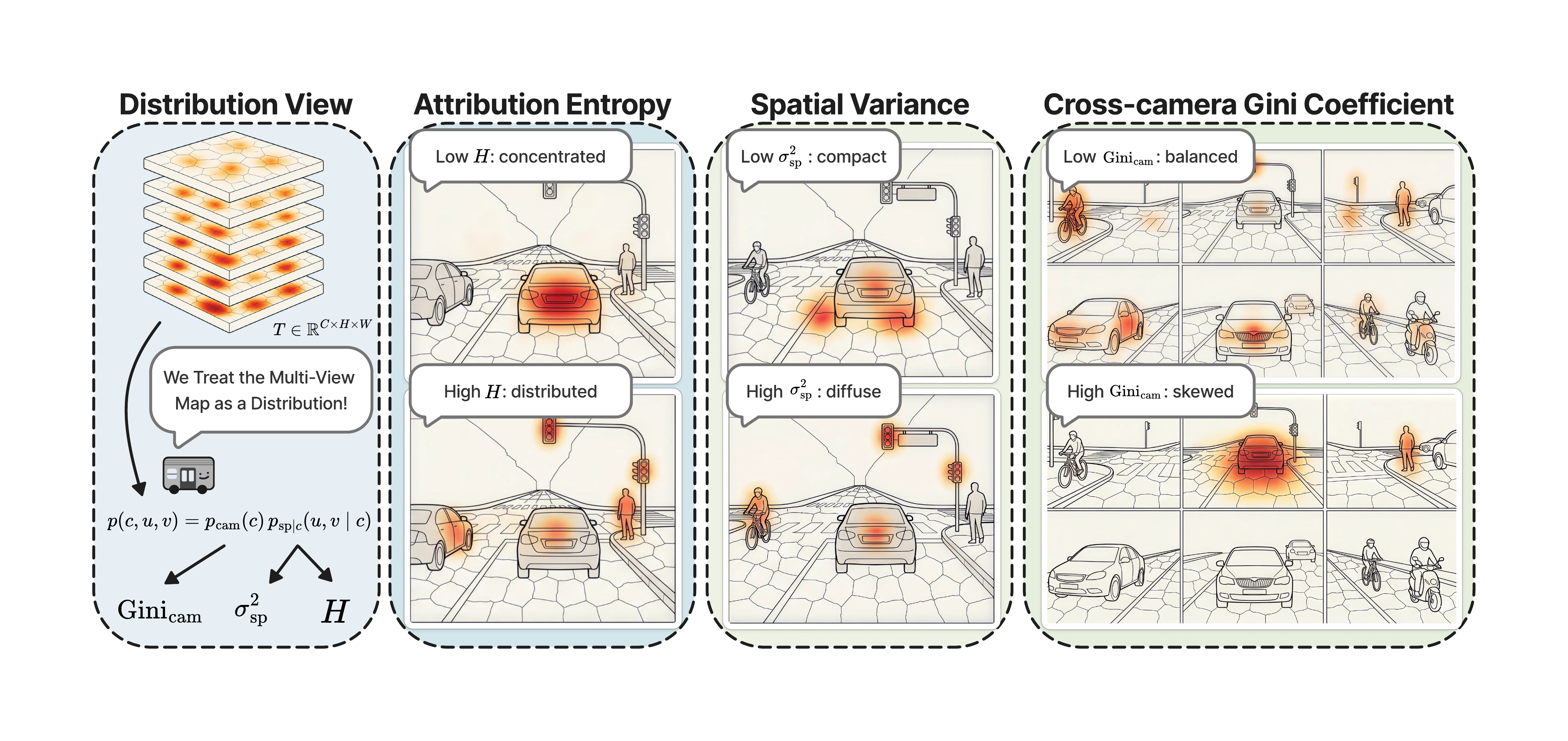}
\caption{\textbf{Three attribution statistics from a distributional view of multi-camera saliency.} Attribution entropy for global concentration, within-camera spatial variance for spatial dispersion within each view, and cross-camera Gini coefficient for imbalance across views.}
\label{fig:three-statistics}
\end{figure*}


We treat the saliency tensor $T\in\mathbb{R}^{C\times H\times W}$ for each sample as a structured attribution distribution over the six-view input. Let $M^{(c)}=|T^{(c)}|\in\mathbb{R}^{H\times W}$ denote the non-negative attribution map for camera $c$. By normalizing attribution weights over all cameras and pixel locations, we obtain a probability distribution $p(c,u,v)$ over the joint multi-camera pixel space, which can be factorized as
\begin{equation}
  p(c,u,v)=p_{\mathrm{cam}}(c)\,p_{\mathrm{sp}\mid c}(u,v\mid c),
  \label{eq:chain-factor}
\end{equation}
where $p_{\mathrm{cam}}(c)=\sum_{u,v}p(c,u,v)$ is the camera-level marginal distribution and $p_{\mathrm{sp}\mid c}(u,v\mid c)$ is the spatial distribution within camera $c$. This factorization naturally separates two levels of attribution structure: how visual reliance is allocated across cameras, and how it is distributed within each camera view.

Based on this structure, we derive three complementary statistics to quantify potential over-reliance in multi-view planning attribution. First, an entropy-based concentration score measures whether attribution is globally concentrated on a small portion of the full six-view pixel space. Second, within-camera spatial variance characterizes whether attribution is localized around specific spatial positions within individual views. Third, the cross-camera Gini coefficient measures whether attribution is disproportionately assigned to one or a few cameras. These statistics capture over-reliance at global, within-view, and cross-view levels, respectively, and provide comparable scalar signals for analyzing planning risk across samples. Our choice of statistics is inspired by prior attribution analyses that quantify distribution concentration, spatial dispersion, and allocation imbalance~\cite{barsbey2025largelr,gu2021lam}, and is adapted here to the multi-camera trajectory planning setting.

\textbf{Attribution entropy.}
The global concentration of attribution reflects whether a planning decision relies on a small portion of the full multi-view visual space. We define attribution entropy as
\begin{equation}
  H=-\sum_{c,u,v} p(c,u,v)\log p(c,u,v) .
  \label{eq:stats-effective-support}
\end{equation}
Here, $H$ measures how uniformly attribution is distributed over the joint pixel space. A smaller $H$ indicates that the planner relies on a limited set of critical visual locations, suggesting stronger global over-reliance.

\textbf{Within-camera spatial variance.}
Conditioned on each camera view, the spatial spread of attribution reflects how localized the planner's visual reliance is within the image plane. We first compute the attribution centroid for camera $c$ as
\[
\bar{\mathbf{r}}^{(c)}=\sum_{u,v}\frac{M^{(c)}_{u,v}}{m_c}\mathbf{r}_{u,v},
\]
where $\mathbf{r}_{u,v}\in\mathbb{R}^2$ is the 2D coordinate of pixel $(u,v)$ and $m_c=\sum_{u,v}M^{(c)}_{u,v}$ is the total attribution mass of camera $c$. The within-camera spatial variance is then defined as
\begin{equation}
  \sigma^2_{\mathrm{sp}}
  =
  \sum_{c}
  \frac{m_c}{\sum_{c'}m_{c'}}
  \sum_{u,v}
  \frac{M^{(c)}_{u,v}}{m_c}
  \left\|
  \mathbf{r}_{u,v}-\bar{\mathbf{r}}^{(c)}
  \right\|^2 .
  \label{eq:stats-spatial-variance}
\end{equation}
A smaller $\sigma^2_{\mathrm{sp}}$ indicates that attribution is concentrated around localized regions within individual views, suggesting stronger within-view over-reliance, whereas a larger value indicates more spatially dispersed evidence.

\textbf{Cross-camera Gini coefficient.}
For a multi-camera planner, attribution allocation across views reflects whether the planning decision depends disproportionately on particular cameras. We quantify this cross-view imbalance with the cross-camera Gini coefficient,
\begin{equation}
  \mathrm{Gini}_{\mathrm{cam}}
  =
  \frac{\sum_{c}\sum_{c'}\left|m_c-m_{c'}\right|}
       {2C\sum_{c}m_c},
  \label{eq:stats-camera-gini}
\end{equation}
where $m_c$ is the total attribution mass of camera $c$. A higher $\mathrm{Gini}_{\mathrm{cam}}$ indicates that attribution is concentrated in a few views, suggesting stronger cross-view over-reliance, whereas a lower value indicates more balanced reliance across cameras.

\section{Experiments}
\label{sec:experiments}

\subsection{Experimental Setup}
\label{sec:exp-setup}

\textbf{Dataset.}
All experiments are conducted in an open-loop setting on the nuScenes validation set~\cite{caesar2020nuscenes}, covering all validation samples across its 150 validation scenes.

\textbf{Planning-risk proxies.}
We define two planning-risk proxies, both computed from the predicted trajectory $\hat y = (\hat y_1, \dots, \hat y_T)$ and the ground-truth annotations. The first is the average displacement error,
\begin{equation}
    \mathrm{ADE}=\frac{1}{T}\sum_{t=1}^{T}\|\hat{y}_t-y_t^*\|_2,
\end{equation}
where $y^*$ is the ground-truth trajectory, ADE serves as a continuous risk proxy. The second is binary, flagging whether the ego footprint $\mathcal{E}(\hat y_t)$ at any predicted waypoint intersects an obstacle bounding box $\mathcal{B}_t^{(k)}$:
\begin{equation}
    \mathrm{collision\_any}
    =
    \mathbf{1}\!\left[
    \exists\, t,k :
    \mathcal{E}(\hat{y}_t) \cap \mathcal{B}_t^{(k)} \neq \emptyset
    \right].
\end{equation}

\textbf{Control variables.}
The association could reflect scene object configuration rather than the planner's visual dependency. We rule this out with three controls derived from nuScenes annotations and multi-camera geometry (visible-object count $N_{\mathrm{obj}}$, spatial spread $D_{\mathrm{obj}}$, and cross-camera imbalance $\mathrm{Gini}_{\mathrm{obj}}$), each structurally matched to one attribution statistic. Formulas are in the appendix.

\textbf{Planners.}
We evaluate three end-to-end planners with markedly different architectures, BridgeAD~\cite{zhang2025bridgead}, UniAD~\cite{hu2023uniad}, and GenAD~\cite{zheng2024genad}, using their publicly released pretrained weights without any modification.

\subsection{Attribution Statistics Track Planning Risk}
\label{sec:exp-association}


We test whether attribution statistics carry stable associations with planning risk proxies. Univariate entries report direct rank correlations or AUROC scores, while combined entries fit ridge regression to ADE and logistic regression to $\mathrm{collision\_any}$. All values in Table~\ref{tab:association-main} are sign-aligned so that larger values indicate higher risk, and $95\%$ confidence intervals from 100 scene-clustered bootstrap resamples are reported in Appendix~\ref{app:ci-association}.

\begin{table}[t]
\centering
\footnotesize
\caption{Association between attribution statistics and planning risk. $\rho_{\text{ADE}}$ denotes the Spearman correlation with ADE, and coll.\ denotes the AUROC for $\mathrm{collision\_any}$.}
\label{tab:association-main}
\setlength{\tabcolsep}{8pt}
\renewcommand{\arraystretch}{1.15}
\begin{tabular}{l |cc |cc |cc}
\toprule
& \multicolumn{2}{c|}{BridgeAD}
& \multicolumn{2}{c|}{UniAD}
& \multicolumn{2}{c}{GenAD} \\
\cmidrule(lr){2-3}\cmidrule(lr){4-5}\cmidrule(lr){6-7}
Feature
  & $\rho_{\text{ADE}}\!\uparrow$ & coll.\ $\!\uparrow$
  & $\rho_{\text{ADE}}\!\uparrow$ & coll.\ $\!\uparrow$
  & $\rho_{\text{ADE}}\!\uparrow$ & coll.\ $\!\uparrow$ \\
\midrule
\multicolumn{7}{l}{\emph{Attribution statistics}} \\
\quad $H$
  & 0.285 & 0.721
  & 0.219 & 0.691
  & 0.035 & 0.606 \\
\quad $\sigma^2_{\mathrm{sp}}$
  & 0.093 & 0.615
  & 0.197 & 0.678
  & 0.298 & 0.718 \\
\quad $\mathrm{Gini}_{\mathrm{cam}}$
  & 0.144 & 0.647
  & 0.128 & 0.628
  & 0.048 & 0.602 \\
\addlinespace[2pt]
\multicolumn{7}{l}{\emph{Object-configuration controls}} \\
\quad $N_{\mathrm{obj}}$
  & 0.109 & 0.609
  & 0.031 & 0.552
  & 0.223 & 0.624 \\
\quad $D_{\mathrm{obj}}$
  & 0.002 & 0.513
  & 0.032 & 0.544
  & 0.043 & 0.530 \\
\quad $\mathrm{Gini}_{\mathrm{obj}}$
  & 0.039 & 0.515
  & 0.003 & 0.505
  & 0.016 & 0.507 \\
\addlinespace[2pt]
\multicolumn{7}{l}{\emph{Joint models}} \\
\quad Controls only
  & 0.137 & 0.611
  & 0.034 & 0.552
  & 0.231 & 0.635 \\
\quad Attribution stats only
  & 0.310 & 0.768
  & 0.299 & 0.770
  & 0.307 & 0.765 \\
\quad Controls $+$ Attribution
  & 0.369 & 0.792
  & 0.308 & 0.773
  & 0.350 & 0.795 \\
\bottomrule
\end{tabular}
\end{table}


\textbf{Different planners route risk through different statistics.} The dominant univariate statistic varies across planners: attribution entropy $H$ leads on BridgeAD and UniAD, within-camera spatial variance $\sigma^2_{\mathrm{sp}}$ leads on GenAD, and the cross-camera Gini coefficient $\mathrm{Gini}_{\mathrm{cam}}$ is consistently weaker but remains sign-aligned with risk. Yet the joint predictive strength is nearly identical across the three planners, with $\rho_{\text{ADE}}$ of $0.310$ / $0.299$ / $0.307$ and AUROC of $0.768$ / $0.770$ / $0.765$ for BridgeAD / UniAD / GenAD, respectively. Which attribution axis carries the strongest signal is architecture-dependent; that such a signal exists is not.


\textbf{The signal is not a stand-in for object-layout complexity.} If attribution were merely re-encoding scene-object configuration, the controls block would already explain most of the variance. It does not. Collision AUROC under controls is at or near chance, ranging from $0.55$ to $0.64$ across the three planners, and $\rho_{\text{ADE}}$ is low on BridgeAD and UniAD and reaches only $0.231$ on GenAD. Attribution statistics raise both metrics far above this on every planner. Adding controls on top of attribution gives a small further gain, so the two sources of information are not redundant, but object configuration is not what drives the attribution signal. The same conclusion holds against ten additional scene-level controls covering ego motion, object kinematics, and perception support, as shown in Appendix~\ref{app:extended-controls}.

\subsection{Generalization to Held-Out Scenes}
\label{sec:exp-out-of-scene}

The previous section fits and evaluates on the same pool of scenes. We now ask whether the relationship transfers to scenes the model has never seen. We split scenes $80\%/20\%$, fit each model on the training scenes, evaluate on the held-out scenes, and average over 20 random splits. Table~\ref{tab:association-ood} reports cross-scene predictive strength for the three joint models. We then ask a sharper question: among the held-out samples, can the predicted score actually pick out the worst ones? To do so, we define the high-risk class as the ADE-top-$10\%$ within each held-out split. Per-feature numbers and confidence intervals are deferred to Appendix~\ref{app:ci-association}.

\begin{table}[t]
\centering
\footnotesize
\caption{Cross-scene predictive strength of the three joint models under an $80\%/20\%$ scene split, averaged over 20 random splits.}
\label{tab:association-ood}
\setlength{\tabcolsep}{6pt}
\renewcommand{\arraystretch}{1.15}
\begin{tabular}{l|cc |cc |cc}
\toprule
& \multicolumn{2}{c}{BridgeAD}
& \multicolumn{2}{c}{UniAD}
& \multicolumn{2}{c}{GenAD} \\
\cmidrule(lr){2-3}\cmidrule(lr){4-5}\cmidrule(lr){6-7}
Models
  & $\rho_{\text{ADE}}\!\uparrow$ & coll.\ $\!\uparrow$
  & $\rho_{\text{ADE}}\!\uparrow$ & coll.\ $\!\uparrow$
  & $\rho_{\text{ADE}}\!\uparrow$ & coll.\ $\!\uparrow$ \\
\midrule
Controls only
  & 0.143 & 0.605
  & 0.058 & 0.552
  & 0.162 & 0.599 \\
Attribution stats only
  & 0.298 & 0.763
  & 0.322 & 0.782
  & 0.343 & 0.779 \\
Controls $+$ Attribution
  & 0.350 & 0.800
  & 0.307 & 0.787
  & 0.345 & 0.799 \\
\bottomrule
\end{tabular}
\end{table}

\textbf{The held-out signal does not collapse.}
Joint attribution statistics on held-out scenes give $\rho_{\text{ADE}}$ of $0.298$ / $0.322$ / $0.343$ for BridgeAD / UniAD / GenAD and AUROC of $0.763$ / $0.782$ / $0.779$, close to the in-domain numbers in Table~\ref{tab:association-main}. The controls-only model, by contrast, drops to $\rho_{\text{ADE}}$ under $0.20$ and AUROC near chance once scenes change. The signal travels with attribution structure, not with an object layout.

\textbf{The score isolates high-ADE samples.}
As shown in Table~\ref{tab:risk-triage}, at $k=10\%$, the budget matches the base rate, so Recall and Precision coincide, both reach $0.40$ to $0.42$ across the three planners, four times the random baseline. Doubling the budget to $20\%$ recovers $0.61$ to $0.64$ of the high-risk samples, tightening it to $5\%$ keeps precision at roughly the same multiple of random. 

\begin{table}[!t]
\centering
\footnotesize
\setlength{\tabcolsep}{6pt}
\renewcommand{\arraystretch}{1.1}
\caption{High-risk identification on held-out scenes at three budget levels. Random is the analytic uniform-sampling baseline.}
\label{tab:risk-triage}
\begin{tabular}{l |cc |cc |cc} 
\toprule
& \multicolumn{2}{c|}{$k = 5\%$}
& \multicolumn{2}{c|}{$k = 10\%$}
& \multicolumn{2}{c}{$k = 20\%$} \\
\cmidrule(lr){2-3}\cmidrule(lr){4-5}\cmidrule(lr){6-7}
Planners
& Recall\,$\uparrow$ & Prec.\,$\uparrow$
& Recall\,$\uparrow$ & Prec.\,$\uparrow$
& Recall\,$\uparrow$ & Prec.\,$\uparrow$ \\
\midrule
BridgeAD & 0.17 & 0.34 & 0.42 & 0.42 & 0.62 & 0.31 \\
UniAD    & 0.22 & 0.44 & 0.40 & 0.40 & 0.61 & 0.31 \\
GenAD    & 0.20 & 0.40 & 0.40 & 0.40 & 0.64 & 0.32 \\
\midrule
Random   & 0.05 & 0.10 & 0.10 & 0.10 & 0.20 & 0.10 \\
\bottomrule
\end{tabular}
\end{table}
A finer-grained quantile analysis in Appendix~\ref{app:quantile-analysis} reaches the same conclusion.

\subsection{Robustness Across Attribution Methods}
\label{sec:exp-robustness}

A natural concern is that the results so far depend on our hierarchical search itself. To test this, we replace it with RISE~\cite{petsiuk2018rise}, a search-free random-sampling baseline whose mechanism shares nothing with our guided coarse-to-fine search, and recompute the three statistics with everything else fixed. If the diagnostic signal survives the swap, it cannot be an artifact of how attributions are computed.

\begin{table}[!t]
\centering
\footnotesize
\caption{Ablation on attribution maps. The proposed hierarchical attribution is replaced with RISE while keeping the same statistics, metrics, and sign convention.}
\label{tab:robustness-rise}
\setlength{\tabcolsep}{7pt}
\renewcommand{\arraystretch}{1.15}
\definecolor{oursblue}{RGB}{220,235,250}
\begin{tabular}{l|cc|cc|cc}
\toprule
& \multicolumn{2}{c|}{BridgeAD}
& \multicolumn{2}{c|}{UniAD}
& \multicolumn{2}{c}{GenAD} \\
\cmidrule(lr){2-3}\cmidrule(lr){4-5}\cmidrule(lr){6-7}
Feature
  & $\rho_{\text{ADE}}\!\uparrow$ & coll.\ $\!\uparrow$
  & $\rho_{\text{ADE}}\!\uparrow$ & coll.\ $\!\uparrow$
  & $\rho_{\text{ADE}}\!\uparrow$ & coll.\ $\!\uparrow$ \\
\midrule
\multicolumn{7}{l}{\emph{RISE univariate}} \\
\quad $H$
  & 0.194 & 0.625
  & 0.191 & 0.630
  & 0.066 & 0.617 \\
\quad $\sigma^2_{\mathrm{sp}}$
  & 0.134 & 0.544
  & 0.103 & 0.512
  & 0.169 & 0.638 \\
\quad $\mathrm{Gini}_{\mathrm{cam}}$
  & 0.109 & 0.552
  & 0.118 & 0.516
  & 0.246 & 0.682 \\
\addlinespace[2pt]
\multicolumn{7}{l}{\emph{Joint model (Attribution stats only)}} \\
\quad RISE
  & 0.205 & 0.642
  & 0.226 & 0.651
  & 0.292 & 0.728 \\
\rowcolor{oursblue}
\quad Ours 
  & \textbf{0.310} & \textbf{0.768}
  & \textbf{0.299} & \textbf{0.770}
  & \textbf{0.307} & \textbf{0.765} \\
\bottomrule
\end{tabular}
\end{table}


As shown in Table~\ref{tab:robustness-rise}, replacing the proposed attribution method with RISE yields joint-model $\rho_{\text{ADE}}$ values of $0.205$ / $0.226$ / $0.292$ and collision AUROC values of $0.642$ / $0.651$ / $0.728$ for BridgeAD / UniAD / GenAD, respectively. All three statistics keep the same sign relative to risk on every planner, although the dominant statistic can shift, for example from $\sigma^2_{\mathrm{sp}}$ to $\mathrm{Gini}_{\mathrm{cam}}$ on GenAD. These results show that the diagnostic signal survives a different attribution algorithm, while the uniformly weaker RISE performance is consistent with the faithfulness gap analyzed in the next section.

\subsection{Faithfulness and Efficiency of Hierarchical Attribution}
\label{sec:exp-fidelity-efficiency}

We now ask which attribution method to actually use. Hierarchical search is compared against RISE~\cite{petsiuk2018rise} as a low-cost baseline and EAGLE~\cite{chen2025eagle} as a faithfulness upper bound that solves~\eqref{eq:ordered} exactly. All three share the same region partition and masking, on a random subset of $1{,}000$ nuScenes validation samples. Faithfulness uses Insertion / Deletion AUC~\cite{petsiuk2018rise} and $\bar{s}_{\mathrm{high}}$~\cite{chen2025vps}, the average planning score at high-retention Insertion budgets, efficiency is wall-clock time per sample.

\begin{table}[t]
    \centering
    \footnotesize
    \caption{Faithfulness and efficiency across different attribution methods. All methods share the same region partition and masking operation.} 
    \label{tab:fidelity-efficiency}
    \setlength{\tabcolsep}{8pt}
    \renewcommand{\arraystretch}{1.15}
    \definecolor{oursblue}{RGB}{220,235,250}
    \begin{tabular}{l l|ccc|c}
    \toprule
    & & \multicolumn{3}{c|}{Faithfulness} & Efficiency \\
    \cmidrule(lr){3-5}\cmidrule(lr){6-6}
    Planners & Methods
      & Insertion\ $\!\uparrow$
      & Deletion\ $\!\downarrow$
      & $\bar{s}_{\mathrm{high}}\!\uparrow$
      & Time (s) $\!\downarrow$ \\
    \midrule
    \multirow{3}{*}{BridgeAD}
      & RISE   & 0.652 & 0.424 & 0.609 & 320.0 \\
      & EAGLE  & \textbf{0.813} & \textbf{0.151} & \textbf{0.845} & 1289.4 \\
    \rowcolor{oursblue}
      & Ours   & 0.769 & 0.210 & 0.775 & \textbf{267.6} \\
    \midrule
    \multirow{3}{*}{UniAD}
      & RISE   & 0.735 & 0.340 & 0.688 & 1539.0 \\
      & EAGLE  & \textbf{0.876} & \textbf{0.153} & \textbf{0.900} & 6360.6 \\
    \rowcolor{oursblue}
      & Ours   & 0.776 & 0.154 & 0.753 & \textbf{1269.8} \\
    \midrule
    \multirow{3}{*}{GenAD}
      & RISE   & 0.594 & 0.252 & 0.483 & 203.4 \\
      & EAGLE  & \textbf{0.836} & \textbf{0.082} & \textbf{0.834} & 658.4 \\
    \rowcolor{oursblue}
      & Ours   & 0.778 & 0.109 & 0.729 & \textbf{173.6} \\
    \bottomrule
    \end{tabular}
\end{table}

\textbf{Faithfulness sits between RISE and EAGLE, much closer to EAGLE.}
On every planner and every faithfulness metric, our method clears RISE by a wide margin and trails EAGLE by a narrow one. On UniAD, Deletion is essentially tied with EAGLE. On the other two planners it lands between the baselines but closer to EAGLE. $\bar{s}_{\mathrm{high}}$ shows the same pattern: hierarchical search recovers not just the final selected set but the order in which critical evidence enters the explanation. The residual gap to EAGLE is the price of giving up exact greedy search.

\textbf{Efficiency below the random-sampling baseline.}
Our method is the fastest on all three planners, taking roughly a fifth of EAGLE's wall-clock time, and slightly faster than RISE, despite RISE being a non-search baseline. Together with the faithfulness numbers above, hierarchical search recovers most of what full greedy search recovers at a small fraction of the computation, which is what makes large-scale attribution statistics tractable.

\section{Conclusion and Discussion}
\label{sec:conclusion}
 

In this paper, we study whether attribution can serve as a predictive signal for planning risk in end-to-end autonomous driving. We propose a hierarchical multi-view attribution framework that treats six camera views as a unified attribution space and uses L2 trajectory consistency to explain continuous planning outputs. Based on the resulting attribution maps, we derive three statistics to characterize global, within-view, and cross-view over-reliance. Experiments on nuScenes with BridgeAD, UniAD, and GenAD show that these attribution-based signals correlate with trajectory error, predict collision risk, and generalize to unseen samples. These results suggest that attribution patterns can provide useful warning signals for decision-level risk in end-to-end driving planners.

\textbf{Limitation and future work.}
The main limitation of our framework is attribution efficiency. Although the proposed hierarchical search improves over exhaustive subset selection, it still requires multiple planner forward passes and cannot yet support real-time risk monitoring. However, the results show that high-fidelity attribution provides effective signals for detecting risky planning decisions. Future work can explore real-time high-fidelity attribution, especially gradient-based or hybrid methods, to better balance computational efficiency and signal reliability for online planning-risk monitoring.

\small{
    \bibliographystyle{unsrt}
	\bibliography{bib/references}}

\clearpage
\appendix

\section{Grouped-Domain Property of the Coarse Stage}
\label{sec:appendix-coarse-stage}

This appendix gives the grouped-domain objective associated with the coarse stage and its basic properties. Notation follows Section~\ref{sec:method} unless stated otherwise.

Let $V$ denote the atomic candidate set, and let
\[
  \mathcal{G}=\{G_1,\dots,G_m\},
  \qquad
  V=\bigsqcup_{r=1}^{m} G_r
\]
be a partition of $V$ into groups. For any $\mathcal{P}\subseteq \mathcal{G}$, define the grouped-domain objective
\[
  F_{\mathrm{grp}}(\mathcal{P})
  =
  F\!\Bigl(\bigcup_{G\in\mathcal{P}} G\Bigr).
\]

\begin{proposition}
\label{prop:coarse-stage-grouped-submodularity}
If $F:2^V\to\mathbb{R}$ is normalized, monotone, and submodular, then $F_{\mathrm{grp}}:2^{\mathcal{G}}\to\mathbb{R}$ is also normalized, monotone, and submodular. Consequently, for any group budget $b\le |\mathcal{G}|$, the solution $\mathcal{P}_b^{\mathrm{greedy}}$ obtained by greedy maximization of $F_{\mathrm{grp}}$ on the grouped domain satisfies
\[
  F_{\mathrm{grp}}(\mathcal{P}_b^{\mathrm{greedy}})
  \ge
  \left(1-\frac{1}{e}\right)
  \max_{\substack{\mathcal{P}\subseteq \mathcal{G}\\ |\mathcal{P}|\le b}}
  F_{\mathrm{grp}}(\mathcal{P}).
\]
\end{proposition}

\begin{proof}
We first show that $F_{\mathrm{grp}}$ inherits the three properties from $F$.

\emph{Normalization.}
\[
  F_{\mathrm{grp}}(\varnothing)
  =
  F\!\Bigl(\bigcup_{G\in\varnothing} G\Bigr)
  =
  F(\varnothing)
  =
  0.
\]

\emph{Monotonicity.} For $\mathcal{P}\subseteq \mathcal{Q}\subseteq \mathcal{G}$,
\[
  \bigcup_{G\in\mathcal{P}} G
  \subseteq
  \bigcup_{G\in\mathcal{Q}} G,
\]
and the monotonicity of $F$ gives
\[
  F_{\mathrm{grp}}(\mathcal{P})
  =
  F\!\Bigl(\bigcup_{G\in\mathcal{P}} G\Bigr)
  \le
  F\!\Bigl(\bigcup_{G\in\mathcal{Q}} G\Bigr)
  =
  F_{\mathrm{grp}}(\mathcal{Q}).
\]

\emph{Submodularity.} Take any $\mathcal{P}\subseteq \mathcal{Q}\subseteq \mathcal{G}$ and any $H\in \mathcal{G}\setminus \mathcal{Q}$, and write
\[
  X=\bigcup_{G\in\mathcal{P}} G,
  \qquad
  Y=\bigcup_{G\in\mathcal{Q}} G,
\]
so that $X\subseteq Y$. The diminishing-returns property of $F$ gives
\[
  F(X\cup H)-F(X)
  \ge
  F(Y\cup H)-F(Y).
\]
Since
\[
  X\cup H
  =
  \bigcup_{G\in\mathcal{P}\cup\{H\}} G,
  \qquad
  Y\cup H
  =
  \bigcup_{G\in\mathcal{Q}\cup\{H\}} G,
\]
this is equivalent to
\[
  F_{\mathrm{grp}}(\mathcal{P}\cup\{H\})-F_{\mathrm{grp}}(\mathcal{P})
  \ge
  F_{\mathrm{grp}}(\mathcal{Q}\cup\{H\})-F_{\mathrm{grp}}(\mathcal{Q}),
\]
so $F_{\mathrm{grp}}$ is submodular.

With $F_{\mathrm{grp}}$ normalized, monotone, and submodular, the classical greedy bound of Nemhauser et al.~\cite{nemhauser1978analysis} applies on the grouped domain and yields
\[
  F_{\mathrm{grp}}(\mathcal{P}_b^{\mathrm{greedy}})
  \ge
  \left(1-\frac{1}{e}\right)
  \max_{\substack{\mathcal{P}\subseteq \mathcal{G}\\ |\mathcal{P}|\le b}}
  F_{\mathrm{grp}}(\mathcal{P}).
\]
\end{proof}

\paragraph{Remark.}
The coarse stage does not change the optimization objective; it only contracts the search domain from atomic regions to groups. When the original $F$ is normalized, monotone, and submodular, the contracted grouped-domain objective inherits these properties, so coarse-stage greedy search retains the classical $(1-1/e)$ approximation guarantee for the corresponding best group subset.

\section{Object-Configuration Control Variables}
\label{sec:appendix-controls}

Section~\ref{sec:exp-setup} introduces three object-configuration control variables, $N_{\mathrm{obj}}$, $D_{\mathrm{obj}}$, and $\mathrm{Gini}_{\mathrm{obj}}$, used to test whether the association between attribution statistics and planning risk is driven by scene composition rather than by the planner's visual-dependency pattern. We give their formal definitions here. Each control is constructed to be a structural counterpart of one attribution statistic in Section~\ref{sec:statistics}: $N_{\mathrm{obj}}$ for the joint-pixel concentration captured by $H$, $D_{\mathrm{obj}}$ for the within-camera spatial spread captured by $\sigma^2_{\mathrm{sp}}$, and $\mathrm{Gini}_{\mathrm{obj}}$ for the cross-camera imbalance captured by $\mathrm{Gini}_{\mathrm{cam}}$.

\paragraph{Setup.}
For each nuScenes validation sample we use all annotated 3D bounding boxes (dynamic and static categories: vehicles, pedestrians, cyclists, traffic obstacles, construction objects, etc.). For an object $k$ with 3D box $\mathcal{B}^{(k)}$ in the ego frame, we project the eight box corners into each of the $C=6$ camera image planes using the provided intrinsics and extrinsics, and define the per-camera 2D footprint $b_k^{(c)}$ as the axis-aligned bounding box of the projected corners that fall in front of the camera. An object is considered visible in camera $c$ if $b_k^{(c)}$ overlaps the $H \times W$ image rectangle by a non-empty area. Partial occlusion is not filtered out: an object whose 2D footprint intersects the image rectangle is counted, regardless of whether other objects sit in front of it.

\paragraph{Object count $N_{\mathrm{obj}}$.}
We count each object once if it is visible in at least one camera:
\begin{equation}
  N_{\mathrm{obj}}
  =
  \bigl|\bigl\{k : \exists\, c \text{ s.t.\ } b_k^{(c)} \cap \mathcal{I}^{(c)} \neq \varnothing \bigr\}\bigr|,
  \label{eq:control-nobj}
\end{equation}
where $\mathcal{I}^{(c)}$ is the image rectangle of camera $c$. As a coarse measure of scene complexity, $N_{\mathrm{obj}}$ plays a role analogous to that of $H$: more objects spread the planner's potential evidence over more candidate locations, just as larger $H$ corresponds to attribution being spread over more pixel locations.

\paragraph{Within-camera spatial spread $D_{\mathrm{obj}}$.}
For each camera $c$, let $\bar{\mathbf{r}}^{(c)}_{\mathrm{obj}}$ denote the mean of the 2D centers of all object footprints visible in that camera, and define the within-camera spatial second moment as the average squared distance from these centers to the mean. We then aggregate across cameras with weights proportional to the per-camera object count $n_c = |\{k : b_k^{(c)}\cap \mathcal{I}^{(c)}\neq \varnothing\}|$:
\begin{equation}
  D_{\mathrm{obj}}
  =
  \sum_{c}
  \frac{n_c}{\sum_{c'}n_{c'}}
  \cdot
  \frac{1}{n_c}
  \sum_{k : b_k^{(c)}\cap \mathcal{I}^{(c)}\neq \varnothing}
  \bigl\|\mathbf{r}_k^{(c)} - \bar{\mathbf{r}}^{(c)}_{\mathrm{obj}}\bigr\|^2,
  \label{eq:control-dobj}
\end{equation}
where $\mathbf{r}_k^{(c)}$ is the 2D center of the footprint $b_k^{(c)}$ in the image plane of camera $c$. Cameras with $n_c=0$ contribute zero. Equation~\eqref{eq:control-dobj} mirrors the within-camera spatial variance $\sigma^2_{\mathrm{sp}}$ in equation~\eqref{eq:stats-spatial-variance}: $\sigma^2_{\mathrm{sp}}$ measures the spread of attribution mass around its centroid in each camera, and $D_{\mathrm{obj}}$ measures the spread of object centers around their centroid in the same image plane, with the same camera-mass aggregation scheme.

\paragraph{Cross-camera imbalance $\mathrm{Gini}_{\mathrm{obj}}$.}
With $n_c$ defined as above, we summarize how object visibility is allocated across the $C$ cameras with the Gini coefficient of $\{n_1,\dots,n_C\}$:
\begin{equation}
  \mathrm{Gini}_{\mathrm{obj}}
  =
  \frac{\sum_{c}\sum_{c'}|n_c-n_{c'}|}
       {2C\sum_{c}n_c}.
  \label{eq:control-gini-obj}
\end{equation}
Equation~\eqref{eq:control-gini-obj} mirrors $\mathrm{Gini}_{\mathrm{cam}}$ in equation~\eqref{eq:stats-camera-gini}: $\mathrm{Gini}_{\mathrm{cam}}$ measures cross-camera imbalance of attribution mass $\{m_c\}$, while $\mathrm{Gini}_{\mathrm{obj}}$ measures cross-camera imbalance of object counts $\{n_c\}$. Higher values of either indicate that activity (attribution mass or object count) concentrates in a small subset of views.

The pairing between each control and its corresponding attribution statistic is by construction: $D_{\mathrm{obj}}$ uses the same camera-mass-weighted spatial-variance template as $\sigma^2_{\mathrm{sp}}$, and $\mathrm{Gini}_{\mathrm{obj}}$ uses the same Gini formula as $\mathrm{Gini}_{\mathrm{cam}}$. This makes the controls maximally favorable to the null hypothesis that attribution structure merely re-encodes object-layout structure: each attribution statistic is benchmarked against a control that targets the same dimension of scene composition.

\section{Confidence Intervals for Main-Text Tables}
\label{app:ci-main}

This appendix reproduces every result reported in the main text together with the $95\%$ scene-clustered bootstrap confidence intervals that were deferred for space. Numbers in parentheses are half-widths of the bootstrap interval, computed over 100 resamples (or 20 random splits, where indicated). Point estimates match those in the corresponding main-text tables.

\subsection{In-Domain Association (Table~\ref{tab:association-main})}
\label{app:ci-association}

Table~\ref{tab:ci-association-main} gives the full version of main-text Table~\ref{tab:association-main}, including the six per-feature univariate rows and all confidence intervals.

\begin{table}[h]
\centering
\footnotesize
\caption{In-domain association with confidence intervals. Same metrics, sign convention, and bold rule as Table~\ref{tab:association-main}. Numbers in parentheses are bootstrap half-widths.}
\label{tab:ci-association-main}
\setlength{\tabcolsep}{4pt}
\renewcommand{\arraystretch}{1.15}
\resizebox{\columnwidth}{!}{%
\begin{tabular}{l cc cc cc}
\toprule
& \multicolumn{2}{c}{BridgeAD}
& \multicolumn{2}{c}{UniAD}
& \multicolumn{2}{c}{GenAD} \\
\cmidrule(lr){2-3}\cmidrule(lr){4-5}\cmidrule(lr){6-7}
Feature
  & $\rho_{\text{ADE}}\!\uparrow$ & coll.\ $\!\uparrow$
  & $\rho_{\text{ADE}}\!\uparrow$ & coll.\ $\!\uparrow$
  & $\rho_{\text{ADE}}\!\uparrow$ & coll.\ $\!\uparrow$ \\
\midrule
\multicolumn{7}{l}{\emph{Attribution statistics}} \\
\quad $H$
  & $0.285{\scriptstyle\,(\pm 0.046)}$ & $0.721{\scriptstyle\,(\pm 0.030)}$
  & $0.219{\scriptstyle\,(\pm 0.056)}$ & $0.691{\scriptstyle\,(\pm 0.040)}$
  & $0.035{\scriptstyle\,(\pm 0.041)}$ & $0.606{\scriptstyle\,(\pm 0.031)}$ \\
\quad $\sigma^2_{\mathrm{sp}}$
  & $0.093{\scriptstyle\,(\pm 0.049)}$ & $0.615{\scriptstyle\,(\pm 0.033)}$
  & $0.197{\scriptstyle\,(\pm 0.076)}$ & $0.678{\scriptstyle\,(\pm 0.044)}$
  & $0.298{\scriptstyle\,(\pm 0.053)}$ & $0.718{\scriptstyle\,(\pm 0.030)}$ \\
\quad $\mathrm{Gini}_{\mathrm{cam}}$
  & $0.144{\scriptstyle\,(\pm 0.059)}$ & $0.647{\scriptstyle\,(\pm 0.034)}$
  & $0.128{\scriptstyle\,(\pm 0.053)}$ & $0.628{\scriptstyle\,(\pm 0.037)}$
  & $0.048{\scriptstyle\,(\pm 0.057)}$ & $0.602{\scriptstyle\,(\pm 0.036)}$ \\
\addlinespace[2pt]
\multicolumn{7}{l}{\emph{Object-configuration controls}} \\
\quad $N_{\mathrm{obj}}$
  & $0.109{\scriptstyle\,(\pm 0.050)}$ & $0.609{\scriptstyle\,(\pm 0.023)}$
  & $0.031{\scriptstyle\,(\pm 0.051)}$ & $0.552{\scriptstyle\,(\pm 0.034)}$
  & $0.223{\scriptstyle\,(\pm 0.082)}$ & $0.624{\scriptstyle\,(\pm 0.114)}$ \\
\quad $D_{\mathrm{obj}}$
  & $0.002{\scriptstyle\,(\pm 0.040)}$ & $0.513{\scriptstyle\,(\pm 0.023)}$
  & $0.032{\scriptstyle\,(\pm 0.046)}$ & $0.544{\scriptstyle\,(\pm 0.029)}$
  & $0.043{\scriptstyle\,(\pm 0.091)}$ & $0.530{\scriptstyle\,(\pm 0.102)}$ \\
\quad $\mathrm{Gini}_{\mathrm{obj}}$
  & $0.039{\scriptstyle\,(\pm 0.048)}$ & $0.515{\scriptstyle\,(\pm 0.020)}$
  & $0.003{\scriptstyle\,(\pm 0.030)}$ & $0.505{\scriptstyle\,(\pm 0.015)}$
  & $0.016{\scriptstyle\,(\pm 0.106)}$ & $0.507{\scriptstyle\,(\pm 0.053)}$ \\
\addlinespace[2pt]
\multicolumn{7}{l}{\emph{Joint models}} \\
\quad Controls only
  & $0.137{\scriptstyle\,(\pm 0.056)}$ & $0.611{\scriptstyle\,(\pm 0.027)}$
  & $0.034{\scriptstyle\,(\pm 0.053)}$ & $0.552{\scriptstyle\,(\pm 0.028)}$
  & $0.231{\scriptstyle\,(\pm 0.192)}$ & $0.635{\scriptstyle\,(\pm 0.103)}$ \\
\quad Attribution stats only
  & $0.310{\scriptstyle\,(\pm 0.054)}$ & $0.768{\scriptstyle\,(\pm 0.036)}$
  & $0.299{\scriptstyle\,(\pm 0.071)}$ & $0.770{\scriptstyle\,(\pm 0.040)}$
  & $0.307{\scriptstyle\,(\pm 0.070)}$ & $0.765{\scriptstyle\,(\pm 0.033)}$ \\
\quad Controls $+$ Attribution
  & $0.369{\scriptstyle\,(\pm 0.070)}$ & $0.792{\scriptstyle\,(\pm 0.044)}$
  & $0.308{\scriptstyle\,(\pm 0.062)}$ & $0.773{\scriptstyle\,(\pm 0.037)}$
  & $0.350{\scriptstyle\,(\pm 0.077)}$ & $0.795{\scriptstyle\,(\pm 0.034)}$ \\
\bottomrule
\end{tabular}%
}
\end{table}

\subsection{Held-Out Association (Table~\ref{tab:association-ood})}
\label{app:ci-association-ood}

The main-text version of Table~\ref{tab:association-ood} reports only the three joint models. Table~\ref{tab:ci-association-ood} adds the six per-feature univariate rows and confidence intervals from 20 random $80\%/20\%$ scene splits. Per-feature ordering is preserved across the in-domain and held-out tables, and the joint-model trends discussed in Section~\ref{sec:exp-out-of-scene} carry over.

\begin{table}[h]
\centering
\footnotesize
\caption{Held-out association with confidence intervals. Numbers in parentheses are half-widths of the $95\%$ bootstrap interval over 20 random scene splits.}
\label{tab:ci-association-ood}
\setlength{\tabcolsep}{4pt}
\renewcommand{\arraystretch}{1.15}
\resizebox{\columnwidth}{!}{%
\begin{tabular}{l cc cc cc}
\toprule
& \multicolumn{2}{c}{BridgeAD}
& \multicolumn{2}{c}{UniAD}
& \multicolumn{2}{c}{GenAD} \\
\cmidrule(lr){2-3}\cmidrule(lr){4-5}\cmidrule(lr){6-7}
Feature
  & $\rho_{\text{ADE}}\!\uparrow$ & coll.\ $\!\uparrow$
  & $\rho_{\text{ADE}}\!\uparrow$ & coll.\ $\!\uparrow$
  & $\rho_{\text{ADE}}\!\uparrow$ & coll.\ $\!\uparrow$ \\
\midrule
\multicolumn{7}{l}{\emph{Attribution statistics}} \\
\quad $H$
  & $0.270{\scriptstyle\,(\pm 0.093)}$ & $0.718{\scriptstyle\,(\pm 0.062)}$
  & $0.204{\scriptstyle\,(\pm 0.106)}$ & $0.692{\scriptstyle\,(\pm 0.071)}$
  & $0.060{\scriptstyle\,(\pm 0.058)}$ & $0.618{\scriptstyle\,(\pm 0.051)}$ \\
\quad $\sigma^2_{\mathrm{sp}}$
  & $0.080{\scriptstyle\,(\pm 0.093)}$ & $0.614{\scriptstyle\,(\pm 0.061)}$
  & $0.220{\scriptstyle\,(\pm 0.130)}$ & $0.692{\scriptstyle\,(\pm 0.093)}$
  & $0.318{\scriptstyle\,(\pm 0.111)}$ & $0.725{\scriptstyle\,(\pm 0.070)}$ \\
\quad $\mathrm{Gini}_{\mathrm{cam}}$
  & $0.148{\scriptstyle\,(\pm 0.129)}$ & $0.652{\scriptstyle\,(\pm 0.062)}$
  & $0.135{\scriptstyle\,(\pm 0.116)}$ & $0.646{\scriptstyle\,(\pm 0.069)}$
  & $0.072{\scriptstyle\,(\pm 0.068)}$ & $0.615{\scriptstyle\,(\pm 0.047)}$ \\
\addlinespace[2pt]
\multicolumn{7}{l}{\emph{Object-configuration controls}} \\
\quad $N_{\mathrm{obj}}$
  & $0.113{\scriptstyle\,(\pm 0.097)}$ & $0.572{\scriptstyle\,(\pm 0.051)}$
  & $0.074{\scriptstyle\,(\pm 0.090)}$ & $0.566{\scriptstyle\,(\pm 0.066)}$
  & $0.156{\scriptstyle\,(\pm 0.095)}$ & $0.586{\scriptstyle\,(\pm 0.060)}$ \\
\quad $D_{\mathrm{obj}}$
  & $0.062{\scriptstyle\,(\pm 0.086)}$ & $0.533{\scriptstyle\,(\pm 0.044)}$
  & $0.080{\scriptstyle\,(\pm 0.089)}$ & $0.559{\scriptstyle\,(\pm 0.052)}$
  & $0.055{\scriptstyle\,(\pm 0.057)}$ & $0.522{\scriptstyle\,(\pm 0.069)}$ \\
\quad $\mathrm{Gini}_{\mathrm{obj}}$
  & $0.057{\scriptstyle\,(\pm 0.078)}$ & $0.526{\scriptstyle\,(\pm 0.034)}$
  & $0.050{\scriptstyle\,(\pm 0.049)}$ & $0.526{\scriptstyle\,(\pm 0.036)}$
  & $0.041{\scriptstyle\,(\pm 0.106)}$ & $0.497{\scriptstyle\,(\pm 0.080)}$ \\
\addlinespace[2pt]
\multicolumn{7}{l}{\emph{Joint models}} \\
\quad Controls only
  & $0.143{\scriptstyle\,(\pm 0.048)}$ & $0.605{\scriptstyle\,(\pm 0.051)}$
  & $0.058{\scriptstyle\,(\pm 0.073)}$ & $0.552{\scriptstyle\,(\pm 0.051)}$
  & $0.162{\scriptstyle\,(\pm 0.022)}$ & $0.599{\scriptstyle\,(\pm 0.092)}$ \\
\quad Attribution stats only
  & $0.298{\scriptstyle\,(\pm 0.088)}$ & $0.763{\scriptstyle\,(\pm 0.060)}$
  & $0.322{\scriptstyle\,(\pm 0.102)}$ & $0.782{\scriptstyle\,(\pm 0.091)}$
  & $0.343{\scriptstyle\,(\pm 0.102)}$ & $0.779{\scriptstyle\,(\pm 0.073)}$ \\
\quad Controls $+$ Attribution
  & $0.350{\scriptstyle\,(\pm 0.181)}$ & $0.800{\scriptstyle\,(\pm 0.090)}$
  & $0.307{\scriptstyle\,(\pm 0.167)}$ & $0.787{\scriptstyle\,(\pm 0.073)}$
  & $0.345{\scriptstyle\,(\pm 0.134)}$ & $0.799{\scriptstyle\,(\pm 0.060)}$ \\
\bottomrule
\end{tabular}%
}
\end{table}

\subsection{Held-Out Risk Triage (Table inset in Section~\ref{sec:exp-out-of-scene})}
\label{app:ci-triage}

Table~\ref{tab:ci-triage} adds confidence intervals to the inline risk-triage table from the main text. Half-widths are computed over the 20 random scene splits.

\begin{table}[h]
\centering
\footnotesize
\caption{Held-out high-ADE identification with confidence intervals. Same setup and reading direction as the inline table in Section~\ref{sec:exp-out-of-scene}. Random is the analytic uniform-sampling expectation and has no sampling variance.}
\label{tab:ci-triage}
\setlength{\tabcolsep}{4pt}
\renewcommand{\arraystretch}{1.15}
\resizebox{\columnwidth}{!}{%
\begin{tabular}{l cc cc cc}
\toprule
& \multicolumn{2}{c}{$k = 5\%$}
& \multicolumn{2}{c}{$k = 10\%$}
& \multicolumn{2}{c}{$k = 20\%$} \\
\cmidrule(lr){2-3}\cmidrule(lr){4-5}\cmidrule(lr){6-7}
Planner
   & Recall $\uparrow$ & Prec. $\uparrow$
  & Recall $\uparrow$ & Prec. $\uparrow$
  & Recall $\uparrow$ & Prec. $\uparrow$ \\
\midrule
BridgeAD
  & $0.17{\scriptstyle\,(\pm 0.13)}$ & $0.34{\scriptstyle\,(\pm 0.26)}$
  & $0.42{\scriptstyle\,(\pm 0.12)}$ & $0.42{\scriptstyle\,(\pm 0.12)}$
  & $0.62{\scriptstyle\,(\pm 0.10)}$ & $0.31{\scriptstyle\,(\pm 0.05)}$ \\
UniAD
  & $0.22{\scriptstyle\,(\pm 0.07)}$ & $0.44{\scriptstyle\,(\pm 0.14)}$
  & $0.40{\scriptstyle\,(\pm 0.15)}$ & $0.40{\scriptstyle\,(\pm 0.15)}$
  & $0.61{\scriptstyle\,(\pm 0.06)}$ & $0.31{\scriptstyle\,(\pm 0.03)}$ \\
GenAD
  & $0.20{\scriptstyle\,(\pm 0.11)}$ & $0.40{\scriptstyle\,(\pm 0.22)}$
  & $0.40{\scriptstyle\,(\pm 0.09)}$ & $0.40{\scriptstyle\,(\pm 0.09)}$
  & $0.64{\scriptstyle\,(\pm 0.11)}$ & $0.32{\scriptstyle\,(\pm 0.06)}$ \\
\midrule
Random
  & $0.05$ & $0.10$ & $0.10$ & $0.10$ & $0.20$ & $0.10$ \\
\bottomrule
\end{tabular}%
}
\end{table}

\subsection{RISE Robustness (Table~\ref{tab:robustness-rise})}
\label{app:ci-rise}

Table~\ref{tab:ci-rise} adds confidence intervals to the RISE robustness experiment.

\begin{table}[h]
\centering
\footnotesize
\caption{RISE robustness with confidence intervals. Same metrics and sign convention as Table~\ref{tab:robustness-rise}.}
\label{tab:ci-rise}
\setlength{\tabcolsep}{4pt}
\renewcommand{\arraystretch}{1.15}
\resizebox{\columnwidth}{!}{%
\begin{tabular}{l cc cc cc}
\toprule
& \multicolumn{2}{c}{BridgeAD}
& \multicolumn{2}{c}{UniAD}
& \multicolumn{2}{c}{GenAD} \\
\cmidrule(lr){2-3}\cmidrule(lr){4-5}\cmidrule(lr){6-7}
Feature
  & $\rho_{\text{ADE}}\!\uparrow$ & coll.\ $\!\uparrow$
  & $\rho_{\text{ADE}}\!\uparrow$ & coll.\ $\!\uparrow$
  & $\rho_{\text{ADE}}\!\uparrow$ & coll.\ $\!\uparrow$ \\
\midrule
\multicolumn{7}{l}{\emph{RISE univariate}} \\
\quad $H$
  & $0.194{\scriptstyle\,(\pm 0.084)}$ & $0.625{\scriptstyle\,(\pm 0.058)}$
  & $0.191{\scriptstyle\,(\pm 0.061)}$ & $0.630{\scriptstyle\,(\pm 0.078)}$
  & $0.066{\scriptstyle\,(\pm 0.058)}$ & $0.617{\scriptstyle\,(\pm 0.120)}$ \\
\quad $\sigma^2_{\mathrm{sp}}$
  & $0.134{\scriptstyle\,(\pm 0.153)}$ & $0.544{\scriptstyle\,(\pm 0.083)}$
  & $0.103{\scriptstyle\,(\pm 0.133)}$ & $0.512{\scriptstyle\,(\pm 0.103)}$
  & $0.169{\scriptstyle\,(\pm 0.111)}$ & $0.638{\scriptstyle\,(\pm 0.121)}$ \\
\quad $\mathrm{Gini}_{\mathrm{cam}}$
  & $0.109{\scriptstyle\,(\pm 0.093)}$ & $0.552{\scriptstyle\,(\pm 0.135)}$
  & $0.118{\scriptstyle\,(\pm 0.122)}$ & $0.516{\scriptstyle\,(\pm 0.108)}$
  & $0.246{\scriptstyle\,(\pm 0.068)}$ & $0.682{\scriptstyle\,(\pm 0.095)}$ \\
\addlinespace[2pt]
\multicolumn{7}{l}{\emph{Joint model (Attribution stats only)}} \\
\quad RISE
  & $0.205{\scriptstyle\,(\pm 0.117)}$ & $0.642{\scriptstyle\,(\pm 0.067)}$
  & $0.226{\scriptstyle\,(\pm 0.083)}$ & $0.651{\scriptstyle\,(\pm 0.063)}$
  & $0.292{\scriptstyle\,(\pm 0.102)}$ & $0.728{\scriptstyle\,(\pm 0.109)}$ \\
\bottomrule
\end{tabular}%
}
\end{table}

\subsection{Faithfulness and Efficiency (Table~\ref{tab:fidelity-efficiency})}
\label{app:ci-fidelity}

Table~\ref{tab:ci-fidelity} adds confidence intervals to the faithfulness and efficiency comparison. The Time column reports raw averages without bootstrap; its variation across runs is small and not the subject of the comparison.

\begin{table}[h]
\centering
\footnotesize
\caption{Faithfulness and efficiency with confidence intervals. Same metrics and bold rule as Table~\ref{tab:fidelity-efficiency}.}
\label{tab:ci-fidelity}
\setlength{\tabcolsep}{6pt}
\renewcommand{\arraystretch}{1.15}
\begin{tabular}{l l ccc c}
\toprule
& & \multicolumn{3}{c}{Faithfulness} & Efficiency \\
\cmidrule(lr){3-5}\cmidrule(lr){6-6}
Planner & Method
  & Ins.\ $\!\uparrow$
  & Del.\ $\!\downarrow$
  & $\bar{s}_{\mathrm{high}}\!\uparrow$
  & Time (s) $\!\downarrow$ \\
\midrule
\multirow{3}{*}{BridgeAD}
  & RISE
    & $0.652{\scriptstyle\,(\pm 0.129)}$
    & $0.424{\scriptstyle\,(\pm 0.144)}$
    & $0.609{\scriptstyle\,(\pm 0.177)}$
    & $\phantom{0}320.0$ \\
  & EAGLE
    & $\mathbf{0.813}{\scriptstyle\,(\pm 0.026)}$
    & $\mathbf{0.151}{\scriptstyle\,(\pm 0.019)}$
    & $\mathbf{0.845}{\scriptstyle\,(\pm 0.049)}$
    & $1289.4$ \\
  & Ours
    & $0.769{\scriptstyle\,(\pm 0.094)}$
    & $0.210{\scriptstyle\,(\pm 0.058)}$
    & $0.775{\scriptstyle\,(\pm 0.111)}$
    & $\mathbf{\phantom{0}267.6}$ \\
\addlinespace[2pt]
\multirow{3}{*}{UniAD}
  & RISE
    & $0.735{\scriptstyle\,(\pm 0.033)}$
    & $0.340{\scriptstyle\,(\pm 0.045)}$
    & $0.688{\scriptstyle\,(\pm 0.065)}$
    & $1539.0$ \\
  & EAGLE
    & $\mathbf{0.876}{\scriptstyle\,(\pm 0.019)}$
    & $\mathbf{0.153}{\scriptstyle\,(\pm 0.026)}$
    & $\mathbf{0.900}{\scriptstyle\,(\pm 0.032)}$
    & $6360.6$ \\
  & Ours
    & $0.776{\scriptstyle\,(\pm 0.026)}$
    & $0.154{\scriptstyle\,(\pm 0.031)}$
    & $0.753{\scriptstyle\,(\pm 0.051)}$
    & $\mathbf{1269.8}$ \\
\addlinespace[2pt]
\multirow{3}{*}{GenAD}
  & RISE
    & $0.594{\scriptstyle\,(\pm 0.139)}$
    & $0.252{\scriptstyle\,(\pm 0.087)}$
    & $0.483{\scriptstyle\,(\pm 0.204)}$
    & $\phantom{0}203.4$ \\
  & EAGLE
    & $\mathbf{0.836}{\scriptstyle\,(\pm 0.057)}$
    & $\mathbf{0.082}{\scriptstyle\,(\pm 0.042)}$
    & $\mathbf{0.834}{\scriptstyle\,(\pm 0.067)}$
    & $\phantom{0}658.4$ \\
  & Ours
    & $0.778{\scriptstyle\,(\pm 0.077)}$
    & $0.109{\scriptstyle\,(\pm 0.030)}$
    & $0.729{\scriptstyle\,(\pm 0.114)}$
    & $\mathbf{\phantom{0}173.6}$ \\
\bottomrule
\end{tabular}
\end{table}

\section{High-ADE Identification by Quantile Thresholds}
\label{app:quantile-analysis}

\subsection{In-Domain Quantile Identification}
\label{app:quantile-indomain}

The Spearman $\rho$ and AUROC metrics in Section~\ref{sec:exp-association} capture monotonic association across the full ADE distribution. A complementary question is whether attribution statistics can isolate the small fraction of samples whose ADE is sharply elevated. We binarize ADE at three thresholds, taking the top $10\%$, $20\%$, and $30\%$ as positive, and report AUROC in the same univariate-versus-joint setup as the main table. Table~\ref{tab:quantile-indomain} gives the in-domain numbers.

The dominant univariate predictor stays the same as in the main analysis at every threshold: $H$ leads on BridgeAD and UniAD, $\sigma^2_{\mathrm{sp}}$ leads on GenAD. As the threshold relaxes from $10\%$ to $30\%$, AUROC drops uniformly because the positive class begins to absorb mid-risk samples and the discrimination task itself becomes harder. The joint model keeps a clear lead over its components throughout.

\begin{table}[h]
\centering
\footnotesize
\caption{In-domain AUROC for identifying high-ADE samples at three thresholds. Positive class is the ADE-top-$k\%$ within the validation pool. Numbers in parentheses are bootstrap half-widths over 100 resamples.}
\label{tab:quantile-indomain}
\setlength{\tabcolsep}{6pt}
\renewcommand{\arraystretch}{1.15}
\begin{tabular}{l ccc}
\toprule
Feature & BridgeAD & UniAD & GenAD \\
\midrule
\multicolumn{4}{l}{\emph{Threshold $k = 10\%$}} \\
\quad $H$
  & $0.767{\scriptstyle\,(\pm 0.042)}$
  & $0.735{\scriptstyle\,(\pm 0.038)}$
  & $0.529{\scriptstyle\,(\pm 0.047)}$ \\
\quad $\sigma^2_{\mathrm{sp}}$
  & $0.596{\scriptstyle\,(\pm 0.051)}$
  & $0.565{\scriptstyle\,(\pm 0.056)}$
  & $0.720{\scriptstyle\,(\pm 0.038)}$ \\
\quad $\mathrm{Gini}_{\mathrm{cam}}$
  & $0.673{\scriptstyle\,(\pm 0.047)}$
  & $0.632{\scriptstyle\,(\pm 0.045)}$
  & $0.550{\scriptstyle\,(\pm 0.049)}$ \\
\quad Joint
  & $0.820{\scriptstyle\,(\pm 0.028)}$
  & $0.772{\scriptstyle\,(\pm 0.033)}$
  & $0.748{\scriptstyle\,(\pm 0.039)}$ \\
\addlinespace[2pt]
\multicolumn{4}{l}{\emph{Threshold $k = 20\%$}} \\
\quad $H$
  & $0.687{\scriptstyle\,(\pm 0.038)}$
  & $0.681{\scriptstyle\,(\pm 0.040)}$
  & $0.532{\scriptstyle\,(\pm 0.041)}$ \\
\quad $\sigma^2_{\mathrm{sp}}$
  & $0.560{\scriptstyle\,(\pm 0.036)}$
  & $0.608{\scriptstyle\,(\pm 0.051)}$
  & $0.690{\scriptstyle\,(\pm 0.037)}$ \\
\quad $\mathrm{Gini}_{\mathrm{cam}}$
  & $0.601{\scriptstyle\,(\pm 0.045)}$
  & $0.594{\scriptstyle\,(\pm 0.038)}$
  & $0.544{\scriptstyle\,(\pm 0.049)}$ \\
\quad Joint
  & $0.714{\scriptstyle\,(\pm 0.036)}$
  & $0.734{\scriptstyle\,(\pm 0.034)}$
  & $0.722{\scriptstyle\,(\pm 0.036)}$ \\
\addlinespace[2pt]
\multicolumn{4}{l}{\emph{Threshold $k = 30\%$}} \\
\quad $H$
  & $0.669{\scriptstyle\,(\pm 0.034)}$
  & $0.637{\scriptstyle\,(\pm 0.037)}$
  & $0.539{\scriptstyle\,(\pm 0.036)}$ \\
\quad $\sigma^2_{\mathrm{sp}}$
  & $0.536{\scriptstyle\,(\pm 0.032)}$
  & $0.612{\scriptstyle\,(\pm 0.048)}$
  & $0.666{\scriptstyle\,(\pm 0.032)}$ \\
\quad $\mathrm{Gini}_{\mathrm{cam}}$
  & $0.571{\scriptstyle\,(\pm 0.038)}$
  & $0.570{\scriptstyle\,(\pm 0.038)}$
  & $0.529{\scriptstyle\,(\pm 0.037)}$ \\
\quad Joint
  & $0.676{\scriptstyle\,(\pm 0.033)}$
  & $0.686{\scriptstyle\,(\pm 0.036)}$
  & $0.685{\scriptstyle\,(\pm 0.032)}$ \\
\bottomrule
\end{tabular}
\end{table}

\subsection{Cross-Scene Quantile Identification}
\label{app:oos-quantile-analysis}

The numbers in Appendix~\ref{app:quantile-analysis} reuse the same scene pool for fitting and evaluation. We repeat the analysis under an $80\%/20\%$ scene split: coefficients are estimated only on training scenes, and AUROC is computed on the held-out scenes for the same three thresholds (Table~\ref{tab:quantile-oos}).

Discrimination of high-ADE samples carries over to held-out scenes. At the strictest threshold $k = 10\%$, $H$ reaches $0.755$ on BridgeAD and $0.731$ on UniAD, and $\sigma^2_{\mathrm{sp}}$ reaches $0.716$ on GenAD, all close to their in-domain counterparts. The joint model attains $0.809$ / $0.776$ / $0.759$ for BridgeAD / UniAD / GenAD, again close to the in-domain values. Across the three thresholds the qualitative ordering of features is unchanged.

\begin{table}[h]
\centering
\footnotesize
\caption{Held-out AUROC under an $80\%/20\%$ scene split, at the same three ADE thresholds as Table~\ref{tab:quantile-indomain}. Positive class is the ADE-top-$k\%$ within each held-out split. Numbers in parentheses are bootstrap half-widths over 100 resamples.}
\label{tab:quantile-oos}
\setlength{\tabcolsep}{6pt}
\renewcommand{\arraystretch}{1.15}
\begin{tabular}{l ccc}
\toprule
Feature & BridgeAD & UniAD & GenAD \\
\midrule
\multicolumn{4}{l}{\emph{Threshold $k = 10\%$}} \\
\quad $H$
  & $0.755{\scriptstyle\,(\pm 0.088)}$
  & $0.731{\scriptstyle\,(\pm 0.080)}$
  & $0.587{\scriptstyle\,(\pm 0.075)}$ \\
\quad $\sigma^2_{\mathrm{sp}}$
  & $0.589{\scriptstyle\,(\pm 0.086)}$
  & $0.581{\scriptstyle\,(\pm 0.104)}$
  & $0.716{\scriptstyle\,(\pm 0.082)}$ \\
\quad $\mathrm{Gini}_{\mathrm{cam}}$
  & $0.675{\scriptstyle\,(\pm 0.113)}$
  & $0.630{\scriptstyle\,(\pm 0.086)}$
  & $0.560{\scriptstyle\,(\pm 0.073)}$ \\
\quad Joint
  & $0.809{\scriptstyle\,(\pm 0.057)}$
  & $0.776{\scriptstyle\,(\pm 0.047)}$
  & $0.759{\scriptstyle\,(\pm 0.055)}$ \\
\addlinespace[2pt]
\multicolumn{4}{l}{\emph{Threshold $k = 20\%$}} \\
\quad $H$
  & $0.681{\scriptstyle\,(\pm 0.079)}$
  & $0.676{\scriptstyle\,(\pm 0.100)}$
  & $0.551{\scriptstyle\,(\pm 0.045)}$ \\
\quad $\sigma^2_{\mathrm{sp}}$
  & $0.567{\scriptstyle\,(\pm 0.045)}$
  & $0.626{\scriptstyle\,(\pm 0.105)}$
  & $0.693{\scriptstyle\,(\pm 0.065)}$ \\
\quad $\mathrm{Gini}_{\mathrm{cam}}$
  & $0.618{\scriptstyle\,(\pm 0.078)}$
  & $0.611{\scriptstyle\,(\pm 0.078)}$
  & $0.555{\scriptstyle\,(\pm 0.052)}$ \\
\quad Joint
  & $0.715{\scriptstyle\,(\pm 0.070)}$
  & $0.741{\scriptstyle\,(\pm 0.069)}$
  & $0.731{\scriptstyle\,(\pm 0.047)}$ \\
\addlinespace[2pt]
\multicolumn{4}{l}{\emph{Threshold $k = 30\%$}} \\
\quad $H$
  & $0.665{\scriptstyle\,(\pm 0.074)}$
  & $0.639{\scriptstyle\,(\pm 0.096)}$
  & $0.554{\scriptstyle\,(\pm 0.055)}$ \\
\quad $\sigma^2_{\mathrm{sp}}$
  & $0.539{\scriptstyle\,(\pm 0.050)}$
  & $0.624{\scriptstyle\,(\pm 0.088)}$
  & $0.672{\scriptstyle\,(\pm 0.063)}$ \\
\quad $\mathrm{Gini}_{\mathrm{cam}}$
  & $0.588{\scriptstyle\,(\pm 0.080)}$
  & $0.587{\scriptstyle\,(\pm 0.074)}$
  & $0.553{\scriptstyle\,(\pm 0.054)}$ \\
\quad Joint
  & $0.678{\scriptstyle\,(\pm 0.070)}$
  & $0.698{\scriptstyle\,(\pm 0.089)}$
  & $0.700{\scriptstyle\,(\pm 0.075)}$ \\
\bottomrule
\end{tabular}
\end{table}

\section{Empirical Correlation Among Attribution Statistics}
\label{app:stat-correlation}

This appendix supplies the empirical correlation analysis referenced in Section~\ref{sec:statistics}. We compute pairwise Spearman correlations among the three attribution statistics on all nuScenes validation samples, separately for each planner.

\begin{figure}[h]
\centering
\includegraphics[width=\linewidth]{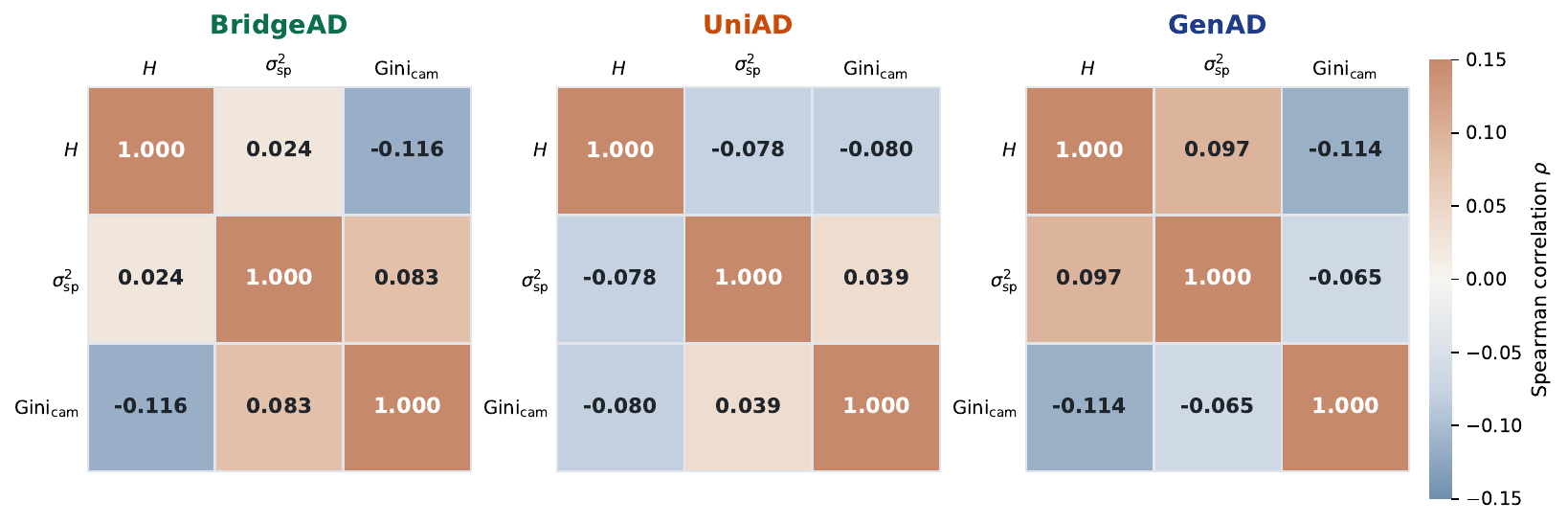}
\caption{Pairwise Spearman correlation among the three attribution statistics, computed on all nuScenes validation samples, shown separately for BridgeAD, UniAD, and GenAD. All off-diagonal entries satisfy $|\rho| \le 0.12$ across the three planners. Correlations are computed on the raw statistic values (without sign alignment to risk).}
\label{fig:stat-correlation}
\end{figure}

Across all three planners, every off-diagonal entry of the correlation matrix satisfies $|\rho| \le 0.12$ (Figure~\ref{fig:stat-correlation}). The largest absolute values arise between $H$ and $\mathrm{Gini}_{\mathrm{cam}}$ on BridgeAD ($\rho = -0.116$) and GenAD ($\rho = -0.114$). These two statistics are the only pair with a structural connection: by the chain rule $H = H(C) + \mathbb{E}_c[H(U,V\mid c)]$, where $H(C)$ is itself a summary of $p_{\mathrm{cam}}(c)$. The reason this connection does not produce a strong empirical correlation is a dynamic-range gap: $H(C) \le \log C \approx 1.79$ nats for $C=6$ cameras, whereas $\mathbb{E}_c[H(U,V\mid c)]$ is on the order of $\log(HW) \approx 14$ nats at the nuScenes image resolution, so the conditional spatial term dominates the sample-level variation of $H$. The other two pairs admit no algebraic dependence, and their correlations stay below $0.10$ on all planners.

The three statistics, while not formally orthogonal as functions of the saliency tensor, are therefore empirically near-independent at the sample level. The ridge and logistic models in Section~\ref{sec:exp-association} aggregate three near-independent signal sources, and the joint-model gain over univariate baselines reflects this additive structure rather than redundancy mitigation.

\section{Robustness to Extended Scene Controls}
\label{app:extended-controls}

The three controls used in Section~\ref{sec:exp-association} ($N_{\mathrm{obj}}$, $D_{\mathrm{obj}}$, $\mathrm{Gini}_{\mathrm{obj}}$) are constructed as structural counterparts to the three attribution statistics. A natural follow-up question is whether the attribution-risk association in Table~\ref{tab:association-main} survives a broader and less curated set of scene controls, in particular ones that prior work has flagged as potentially confounding for nuScenes open-loop planning evaluation. This section reports such a stress test using ten additional scene-level controls that cover ego motion, object class composition, near-field density, lateral layout, perception support, and interaction dynamics.

\paragraph{Setup.}
For each nuScenes validation sample we work in the ego ground-plane (BEV) coordinate frame, with the ego vehicle at the origin. Annotated objects are indexed by $i \in \mathcal{O}$. Each object $i$ has a class label $c_i$ from the standard nuScenes taxonomy, a ground-plane center $p_i = (x_i, y_i) \in \mathbb{R}^2$, a distance $d_i = \|p_i\|_2$, an annotated ground-plane velocity $v_i = (v_{x,i}, v_{y,i}) \in \mathbb{R}^2$, a speed $s_i = \|v_i\|_2$, and an associated radar-point count $r_i$ (the number of radar returns linked to the object's bounding box). Let $\mathcal{V}, \mathcal{B} \subseteq \mathcal{O}$ denote the subsets of vehicles and barriers/cones, and let $\mathcal{D} \subseteq \mathcal{O}$ denote the subset of dynamic actors (vehicles, pedestrians, bicycles, motorcycles). Let $v_e \in \mathbb{R}^2$ be the ego ground-plane velocity, read from the CAN bus when available and from the \texttt{ego\_status} field otherwise. All quantities are computed from current-frame nuScenes annotations. We deliberately exclude variables derived from the predicted future trajectory (predicted trajectory length, final displacement, lateral offset, curvature) and from future actor proximity or time-to-collision: such variables are near-surrogates for ADE or $\mathrm{collision\_any}$ rather than independent scene controls, and including them would inflate apparent control performance for reasons unrelated to scene structure.

\paragraph{Ego speed.}
The current ego translational speed,
\begin{equation}
  \mathrm{ego\_speed} = \|v_e\|_2,
  \label{eq:ctrl-ego-speed}
\end{equation}
captures how fast the ego vehicle is currently moving. It is the most direct candidate confound for ADE, since faster ego motion mechanically inflates per-step displacement error.

\paragraph{Barrier and cone count.}
The number of static traffic obstacles in the scene,
\begin{equation}
  \mathrm{barrier\_cone\_count} = |\mathcal{B}|,
  \label{eq:ctrl-barrier-count}
\end{equation}
where $\mathcal{B}$ collects nuScenes barrier and traffic-cone categories. This control captures construction or lane-restriction complexity that is not reflected in the dynamic actor count.

\paragraph{Near-field object count.}
The number of annotated objects within five meters of the ego vehicle,
\begin{equation}
  \mathrm{nearfield\_count}
  =
  \sum_{i \in \mathcal{O}} \mathbf{1}\!\bigl[d_i \le 5\bigr],
  \label{eq:ctrl-nearfield}
\end{equation}
captures local clutter regardless of class. The five-meter radius is chosen to match the typical ego footprint plus a one-vehicle margin.

\paragraph{Side object count.}
The number of objects whose bearing $\theta_i = \mathrm{atan2}(y_i, x_i)$ places them mainly to the left or right of ego rather than ahead or behind,
\begin{equation}
  \mathrm{side\_object\_count}
  =
  \sum_{i \in \mathcal{O}} \mathbf{1}\!\Bigl[\tfrac{\pi}{3} < |\theta_i| < \tfrac{2\pi}{3}\Bigr].
  \label{eq:ctrl-side-objects}
\end{equation}
This complements front/rear density and captures lane-change-relevant lateral context. The angular band $[\pi/3, 2\pi/3]$ corresponds to roughly the left and right thirds of the surround.

\paragraph{Mean radar points per object.}
A coarse proxy for sensor support per actor,
\begin{equation}
  \mathrm{mean\_radar\_pts}
  =
  \frac{1}{|\mathcal{O}_{r}|}
  \sum_{i \in \mathcal{O}_{r}} r_i,
  \label{eq:ctrl-mean-radar}
\end{equation}
where $\mathcal{O}_r \subseteq \mathcal{O}$ is the subset of objects with finite radar-point counts. Higher values indicate that the sensor suite is currently producing dense radar returns on annotated actors. Samples with no finite radar counts are reported as missing.

\paragraph{Nearest vehicle speed.}
The speed of the closest annotated vehicle,
\begin{equation}
  \mathrm{nearest\_veh\_speed}
  =
  s_{i^\star},
  \qquad
  i^\star = \arg\min_{i \in \mathcal{V}} d_i,
  \label{eq:ctrl-nearest-veh-speed}
\end{equation}
captures the kinematic state of the most-likely interaction partner. Samples with no annotated vehicle are reported as missing.

\paragraph{Mean object speed within 20 meters.}
The average speed of dynamic actors within a 20-meter radius,
\begin{equation}
  \mathrm{mean\_speed\_20m}
  =
  \frac{1}{|\mathcal{D}_{20}|}
  \sum_{i \in \mathcal{D}_{20}} s_i,
  \qquad
  \mathcal{D}_{20} = \{i \in \mathcal{D} : d_i \le 20\},
  \label{eq:ctrl-mean-speed-20m}
\end{equation}
summarizes the local kinematic activity. Samples with no dynamic actor within 20 meters are reported as missing.

\paragraph{Approaching object count.}
The number of dynamic actors whose relative motion is closing toward ego,
\begin{equation}
  \mathrm{approach\_count}
  =
  \sum_{i \in \mathcal{D}}
  \mathbf{1}\!\bigl[p_i^\top (v_i - v_e) < 0\bigr],
  \label{eq:ctrl-approach-count}
\end{equation}
where $p_i^\top (v_i - v_e) < 0$ is the standard closing-motion criterion: the relative velocity has a component pointing back toward the ego origin. This uses only current-frame quantities and does not access future ground truth.

\paragraph{Approaching vehicle count.}
The same closing-motion criterion restricted to vehicles,
\begin{equation}
  \mathrm{approach\_veh\_count}
  =
  \sum_{i \in \mathcal{V}}
  \mathbf{1}\!\bigl[p_i^\top (v_i - v_e) < 0\bigr].
  \label{eq:ctrl-approach-veh}
\end{equation}
Vehicles are the dominant collision-relevant actor class, and isolating them tests whether the approaching-actor signal lives in the vehicle subset specifically.

\paragraph{Dynamic actor ratio.}
The fraction of annotated objects that are dynamic actor categories,
\begin{equation}
  \mathrm{dyn\_ratio}
  =
  \frac{|\mathcal{D}|}{|\mathcal{O}|},
  \label{eq:ctrl-dyn-ratio}
\end{equation}
captures the composition of the scene independent of its absolute size. Samples with $|\mathcal{O}| = 0$ are reported as missing.

\paragraph{Univariate associations.}
Table~\ref{tab:extended-controls-univariate} reports the sign-aligned Spearman correlation with ADE and the sign-aligned AUROC for $\mathrm{collision\_any}$ for each of the ten controls. Most are weakly associated with planning risk on most planners. ADE associations stay below $0.10$ in the majority of cells, and collision AUROCs largely sit between $0.50$ and $0.65$. Two patterns are worth singling out before turning to the joint model.

\begin{table}[h]
\centering
\footnotesize
\caption{Univariate association of ten extended scene controls with planning-risk proxies. $\rho_{\text{ADE}}$ is the Spearman correlation with ADE. coll.\ is the AUROC for $\mathrm{collision\_any}$. All values are sign-aligned to be positive in the risk direction, matching the convention of Table~\ref{tab:association-main}.}
\label{tab:extended-controls-univariate}
\setlength{\tabcolsep}{6pt}
\renewcommand{\arraystretch}{1.15}
\begin{tabular}{l cc cc cc}
\toprule
& \multicolumn{2}{c}{BridgeAD}
& \multicolumn{2}{c}{UniAD}
& \multicolumn{2}{c}{GenAD} \\
\cmidrule(lr){2-3}\cmidrule(lr){4-5}\cmidrule(lr){6-7}
Control
  & $\rho_{\text{ADE}}\!\uparrow$ & coll.\ $\!\uparrow$
  & $\rho_{\text{ADE}}\!\uparrow$ & coll.\ $\!\uparrow$
  & $\rho_{\text{ADE}}\!\uparrow$ & coll.\ $\!\uparrow$ \\
\midrule
Ego speed
  & $0.061$ & $0.590$
  & $0.399$ & $0.509$
  & $0.007$ & $0.648$ \\
Barrier/cone count
  & $0.090$ & $0.549$
  & $0.021$ & $0.586$
  & $0.054$ & $0.519$ \\
Near-field count ($\le 5$m)
  & $0.034$ & $0.665$
  & $0.004$ & $0.617$
  & $0.044$ & $0.518$ \\
Side object count
  & $0.061$ & $0.629$
  & $0.025$ & $0.669$
  & $0.016$ & $0.551$ \\
Mean radar points/object
  & $0.155$ & $0.519$
  & $0.038$ & $0.529$
  & $0.053$ & $0.573$ \\
Nearest vehicle speed
  & $0.012$ & $0.588$
  & $0.150$ & $0.534$
  & $0.015$ & $0.505$ \\
Mean speed within 20m
  & $0.054$ & $0.543$
  & $0.078$ & $0.528$
  & $0.051$ & $0.603$ \\
Approaching object count
  & $0.079$ & $0.631$
  & $0.013$ & $0.607$
  & $0.014$ & $0.533$ \\
Approaching vehicle count
  & $0.144$ & $0.602$
  & $0.014$ & $0.596$
  & $0.023$ & $0.511$ \\
Dynamic actor ratio
  & $0.057$ & $0.503$
  & $0.081$ & $0.536$
  & $0.021$ & $0.562$ \\
\bottomrule
\end{tabular}
\end{table}

\paragraph{Observation: ego speed is unusually predictive of UniAD ADE.}
Among the ten univariate controls, ego speed shows a notably larger ADE association on UniAD ($\rho_{\text{ADE}} = 0.399$) than on BridgeAD ($0.061$) or GenAD ($0.007$), while no other control on UniAD reaches a comparable magnitude. This pattern is qualitatively consistent with prior critiques of nuScenes open-loop planning evaluation, which argue that planners incorporating ego status, especially ego velocity, can rely heavily on such state cues and may be evaluated under metrics that do not fully separate perception-driven planning from ego-motion extrapolation~\cite{li2024egostatus,zhai2023rethinking}. We treat this as an independent diagnostic observation rather than a claim about UniAD itself: our framework recovers, from attribution structure on the input side, a phenomenon that has previously been raised on the evaluation side. The next paragraph asks the question this observation makes pressing, namely whether the attribution-statistic signal in Table~\ref{tab:association-main} merely re-encodes ego speed.

\paragraph{Joint model: attribution statistics carry information beyond extended controls.}
Table~\ref{tab:extended-controls-joint} reports three joint models: ten controls alone, attribution statistics alone (reproduced from Table~\ref{tab:association-main}), and the union of the two. Three patterns emerge.

First, on BridgeAD and GenAD, the ten-control joint model is weak. $\rho_{\text{ADE}}$ reaches only $0.071$ and $0.025$, and collision AUROC sits at $0.702$ and $0.477$, while attribution statistics alone reach $\rho_{\text{ADE}}$ near $0.31$ and AUROC near $0.77$. Adding the ten controls on top of attribution statistics produces only a small further gain on these two planners. $\rho_{\text{ADE}}$ rises from $0.310$ to $0.332$ on BridgeAD and from $0.307$ to $0.309$ on GenAD. On these two planners, attribution structure absorbs essentially the entire risk-relevant scene signal that the controls can offer.

Second, UniAD departs from this pattern. The ten-control joint reaches $\rho_{\text{ADE}} = 0.405$, exceeding attribution statistics alone ($0.299$). This excess is essentially traceable to ego speed: removing ego speed from the control set drops the joint $\rho_{\text{ADE}}$ to a level comparable to the BridgeAD and GenAD numbers above. UniAD's ADE on nuScenes therefore does carry a substantial ego-motion component that attribution statistics, computed from the visual saliency tensor, do not see.

Third, even on UniAD, attribution statistics provide independent information once ego speed is partialled out. The combined model attains $\rho_{\text{ADE}} = 0.478$, exceeding both the controls-only model ($0.405$) and attribution-only ($0.299$). The same pattern holds for collision AUROC, where the combined model reaches $0.794$ against $0.689$ for controls-only and $0.770$ for attribution-only. The signal that attribution structure carries on UniAD is therefore complementary to ego speed rather than a re-encoding of it.

\begin{table}[h]
\centering
\footnotesize
\caption{Joint models on the ten extended controls, the three attribution statistics, and their union. Attribution-only numbers are reproduced from Table~\ref{tab:association-main}.}
\label{tab:extended-controls-joint}
\setlength{\tabcolsep}{6pt}
\renewcommand{\arraystretch}{1.15}
\begin{tabular}{l cc cc cc}
\toprule
& \multicolumn{2}{c}{BridgeAD}
& \multicolumn{2}{c}{UniAD}
& \multicolumn{2}{c}{GenAD} \\
\cmidrule(lr){2-3}\cmidrule(lr){4-5}\cmidrule(lr){6-7}
Joint model
  & $\rho_{\text{ADE}}\!\uparrow$ & coll.\ $\!\uparrow$
  & $\rho_{\text{ADE}}\!\uparrow$ & coll.\ $\!\uparrow$
  & $\rho_{\text{ADE}}\!\uparrow$ & coll.\ $\!\uparrow$ \\
\midrule
10 controls only
  & $0.071$ & $0.702$
  & $0.405$ & $0.689$
  & $0.025$ & $0.477$ \\
Attribution stats only
  & $0.310$ & $0.768$
  & $0.299$ & $0.770$
  & $0.307$ & $0.765$ \\
10 controls $+$ Attribution stats
  & $0.332$ & $0.802$
  & $0.478$ & $0.794$
  & $0.309$ & $0.772$ \\
\bottomrule
\end{tabular}
\end{table}

Taken together, these results extend the controls analysis in Section~\ref{sec:exp-association}. The risk-relevant signal carried by attribution statistics is not a re-encoding of ego motion, near-field object density, perception support, or interaction dynamics. On two of three planners, attribution structure dominates. On the third, where ego status is itself unusually predictive, attribution structure remains independently informative.

\section{Broader Impacts}
\label{app:broader-impacts}

This work proposes a post-hoc diagnostic framework for 
end-to-end autonomous driving planners. Because the target 
domain is safety-critical, we briefly discuss the dual 
implications of treating attribution as a measurement object 
rather than a per-sample visualization.

\paragraph{Positive impacts.}
The framework supports scalable pre-deployment auditing of 
end-to-end planners without modifying the base model or 
training auxiliary risk classifiers. By exposing 
visual-dependency structure as a small set of cross-sample 
comparable statistics, it makes it possible to flag scenes 
where a planner's reliance on visual evidence is concentrated, 
spatially diffuse, or unevenly distributed across views in 
ways that co-occur with elevated planning risk. This kind of 
input-side diagnostic is complementary to existing 
output-side safety evaluation (trajectory error, collision 
checks) and to closed-loop simulation testing, and may help 
developers and auditors locate failure-prone scenarios more 
efficiently than visual inspection alone. The framework also 
makes the way visual-dependency structure shifts with training 
data, model scale, or supervision signal, a measurable 
quantity, which we hope will be useful for the broader 
end-to-end driving research community.

\paragraph{Risks of misinterpretation.}
The most direct risk is that scalar attribution statistics, or their AUROC against a risk proxy is read as safety 
certifications rather than as diagnostic signals. A collision 
AUROC of around $0.77$ on open-loop nuScenes does not imply 
that a planner is safe to deploy, nor that the remaining 
$23\%$ of the ROC area is an acceptable residual risk. The 
statistics measure how a planner allocates its visual 
reliance, not whether the resulting trajectory is correct, 
and the association we report is a population-level 
statistical relationship rather than a per-sample guarantee. 
Operational decisions that depend on this signal should be 
paired with independent safety validation, including 
closed-loop testing, scenario-based verification, and 
domain-specific safety cases.

\paragraph{Scope and mitigation.}
Several aspects of the framework limit how far its conclusions 
should be carried. The empirical evidence is drawn from 
open-loop evaluation on nuScenes with three specific 
pretrained planners, and we do not claim that the same attribution-risk association holds for arbitrary planner architectures, sensor configurations, or operational design 
domains. Compressing a multi-camera saliency tensor into 
three scalar statistics also discards spatial structure and 
higher-order shape information, so the absence of a flagged 
anomaly in our statistics does not rule out unreliable visual 
reliance along axes the statistics do not capture. We frame 
the framework throughout as diagnostic rather than 
certificatory, report all results with scene-clustered 
bootstrap confidence intervals, and recommend that any 
operational use be paired with independent closed-loop 
safety validation rather than substituted for it.

\paragraph{Data and assets.}
All experiments use the publicly released nuScenes validation 
set and publicly released pretrained weights of the three 
planners, with no fine-tuning or modification. The work 
introduces no new data collection, no new deployed system, 
and no model release with elevated misuse potential.

\end{document}